\documentclass[english,10pt]{article}
\usepackage[papersize={19cm,25cm},body={15cm,21cm},centering]{geometry}
\usepackage[T1]{fontenc}
\usepackage[latin9]{inputenc}
\usepackage{amsmath}
\usepackage{amssymb}
\usepackage{url}
\usepackage{babel}
 \usepackage{graphicx}
\usepackage{epsfig,subfigure,color}
\usepackage{graphics}
 \usepackage{graphicx}
\usepackage{amsmath,amsfonts,amsthm,latexsym}
\usepackage{mathrsfs,amsbsy}
\usepackage{multirow}
\usepackage{epstopdf}
\usepackage{algorithm}
\usepackage{algorithmic}
\usepackage{cite}

\newtheorem{theorem}{Theorem}
\newtheorem{lemma}{Lemma}
\newtheorem{property}{Property}
\newtheorem{example}{Example}
\newtheorem{remark}{Remark}
\newtheorem{proposition} {Proposition} 

\theoremstyle{definition}
\newtheorem{defn}{Definition}
\usepackage{bm}

 \usepackage{color}

\DeclareMathOperator{\rank}{rank}
\DeclareMathOperator{\mat}{Mat}
\DeclareMathOperator{\tr}{trace}

\usepackage{enumitem} 
\begin{document}

\title{Separable Quaternion Matrix Factorization for Polarization Images}
\date{}
\author{
Junjun Pan \qquad  Michael K. Ng\thanks{Department of Mathematics,
The University of Hong Kong.
Emails: junjpan@hku.hk, mng@maths.hku.hk. M. Ng's research is supported
 in part by HKRGC GRF 12300218, 12300519, 17201020 and 17300021. } 
}  
\maketitle

\begin{abstract}

Polarization is a unique characteristic of transverse wave and is represented by Stokes parameters. Analysis of polarization states can reveal valuable information about the sources. In this paper, we propose a separable low-rank quaternion linear mixing model to polarized signals:  we assume each column of source factor matrix $\breve{\mathbf{W}}$ equals a column of polarized data matrix $\breve{\mathbf{M}}$ and refer to the corresponding problem as separable quaternion  matrix factorization (SQMF). We discuss some properties of the matrix that can be decomposed by SQMF. To determine the source factor matrix $\breve{\mathbf{W}}$ in quaternion space, we propose a heuristic algorithm called quaternion successive projection algorithm (QSPA) inspired by the successive projection algorithm. To guarantee the effectiveness of QSPA, a new normalization operator is proposed for the quaternion matrix. We use a block coordinate descent algorithm to compute nonnegative factor matrix $\mathbf{H}$ in real number space. We test our method on the applications of polarization image representation and spectro-polarimetric imaging unmixing to verify its effectiveness.
\end{abstract}


\section{Introduction} 

Polarization is one of the primary characteristics of transverse waves such as optics \cite{born2013principles}, seismology\cite{aki2002quantitative}, radio\cite{weiler1973synthesis}, and microwaves\cite{kosowsky1999introduction}. It is a unique transverse wave phenomenon that describes the geometrical orientation of the oscillations.  Different materials, the Earth's surface, and atmospheric targets in the process of reflection, scattering, and transmission and emission of electromagnetic radiation will produce their own nature of the  polarization characteristics. Analysis of the polarization states imposed by these processes will potentially reveal useful information about the source. Hence it works well in many applications in remote sensing \cite{lee2017polarimetric,tyo2006review},  agriculture\cite{paloscia1988microwave}, astronomy\cite{kovac2002detection,kamionkowski1997statistics} and many other fields.

For example, hyperspectral imaging is a well-known technique collects and processes spectra intensity information from across the electromagnetic spectrum. But due to the limitations of optical resolution, ground structure, noise and other factors, different materials often exhibit "the same spectra" that are difficult to identify. To address this issue, the polarization modulation technique integrated into spectral imaging that can further lead to a more powerful technique, namely spectro-polarimetric imaging\cite{zhao2009spectropolarimetric, miron2011joint,mu2012static, boccaletti2012spices}. 
The spectro-polarimetric imaging data consists of spatial, frequency and polarization information.
 The emerging technique has been widely used in astronomy \cite{antonucci1984optical,rousselet2000spectro}, climatology\cite{kokhanovsky2015space}, medical image \cite{boulbry2006novel}, among others.

The state of polarization can be entirely characterized by the four components of the Stokes vector  \cite{mcmaster1954polarization, berry1977measurement}.  The Stokes parameters were first defined by George Gabriel Stokes in 1852 and used widely to describe the intensity and polarization of light. They can be measured in experiments \cite{gori1999measuring,schaefer2007measuring} and reformulated in quaternion algebra\cite{kuntman2019quaternion}. In the blind polarized source separation problems, low-rank approximation models based on quaternions play an important role and have been studied in many references, see for example \cite{le2004singular,javidi2011fast,via2010quaternion}. 
Most of the research works rely on standard quaternion mixing model assumption, without considering physical properties of light. Until 2020, a quaternion linear mixing model named quaternion nonnegative matrix factorization(QNMF) has been proposed in reference \cite{flamant2020quaternion}, specifically for blind unmixing of spectro-polarimetric data. The QNMF model extends nonnegativity constraints based on  physical properties of light,  and a quaternion alternating least squares (QALS) algorithm is presented accordingly.

Since QNMF is extended from standard NMF, its solution could not be guaranteed to be unique in most cases. However, the uniqueness of a factorization is crucial in practice. For the NMF problem,  additional constraints on the factor matrix are usually imposed to guarantee its uniqueness. Inspired by separable NMF, in this paper, we introduce separability into QNMF model, assuming that each column of source factor matrix  equals a column of the polarized data matrix. We refer to it as separable quaternion matrix factorization (SQMF).

The separability assumption makes sense in practical situations. For example \cite{ma2013signal}, in hyperspectral unmixing, each column of the data matrix is the spectral signature of a pixel. Separability, also known as pure-pixel assumption, means that for each material (i.e., endmember) within the hyperspectral image, there exists a pixel that only contains that material.  
For more separable NMF models, we refer the interested reader to \cite{gillis2020nonnegative,pan2019generalized,pan2021co}.

Similarly, the SQMF model requires in the application of spectro-polarimetric unmixing, that each column of the data matrix consists of the spectral and polarization signature of a pixel. And for any pixels, data can be represented as a linearly weighted combination of several elementary sources. For each source corresponding to a material, there exists a pixel that only contains that source (material).  Here the separability in spectro-polarimetric imaging can be seen as a natural extension of pure-pixel assumption in hyperspectral imaging.

\subsection{Contributions and outline of this paper }

This paper considers the quaternion matrix factorization under separability conditions, referred to as  separable quaternion matrix factorization (SQMF). The contributions are summarized as follows.

1) To investigate the SQMF problem, we present several properties of SQ-matrices that can be decomposed by SQMF. We show the relationship between SQ-matirces and its corresponding intensity (real) component and polarization (imaginary) component. The uniqueness of SQMF has been proven from a linear algebra perspective. 

2) In the SQMF problem, to find the source matrix in quaternion space, we propose an
algorithm inspired by the successive projection algorithm
(SPA) \cite{araujo2001successive, gillis2014fast} which we refer to the quaternion successive projection algorithm (QSPA). It only requires O(mnr)
operations. We demonstrate that QSPA can identify the quaternion-valued  factor matrix correctly for SQ-matrices.  To compute the activation factor matrix in real number space, we provide a simple algorithm extended from hierarchical nonnegative least squares.
% named quaternion hierarchical nonnegative least squares (QHNLS) 

3) We test the proposed method in the polarization image representation and spectro-polarized imaging unmixing. Compared to state-of-the-art techniques, SQMF can capture the main characteristic of polarized data and attain a high approximation to the data. Its computational time is very competitive. 

The rest paper is organized as follows.
In Section 2, we briefly review quaternion,  polarization and Stokes parameters. Section 3 introduces the separable quaternion  matrix factorization problem and presents some related properties. We also discuss the uniqueness of the SQMF problem. 
Section 4 proposes the quaternion successive projection algorithm to determine column subset to form the quaternion factor matrix, and quaternion  hierarchical nonnegative least squares to compute the real factor matrix. Section 5 validates the effectiveness of the proposed method in numerical experiments on realistic polarization image and simulated spectro-polarimetric data set. 
The concluding remarks are given in Section 6.

\section{Preliminaries}

Let us start with notations used throughout this paper.  The real number field, the complex number field and the quaternion algebra are defined by $\mathbb{R}$, $\mathbb{C}$ and $\mathbb{H}$ respectively. Unless otherwise specified,
lowercase letters represent real
numbers, for example, $a\in \mathbb{R}$. The bold lowercase letters represent real vectors, such as, $\mathbf{a}\in \mathbb{R}^n$. Real matrices are denoted by bold capital letters, like $\mathbf{A} \in \mathbb{R}^{m\times n}$. The numbers, vectors, and matrices under the quaternion field are represented by the corresponding symbols wearing a hat, for example $ \breve{a} \in \mathbb{H}$,  $\mathbf{\breve{a}}\in \mathbb{H}^n$ and $\mathbf{\breve{A}}\in \mathbb{H}^{m\times n}$.

\subsection{Quaternion}
Quaternions $\mathbb{H}$ are generally represented in the following Cartesian form,
$$ 
\breve{q}=q_0+\mathtt{i}q_{1}+\mathtt{j}q_2+\mathtt{k}q_3,
$$
where $q_0,q_1,q_2,q_3 \in \mathbb{R}$, and $\mathtt{i},\mathtt{j},\mathtt{k}$ are imaginary units such that 
 $$\mathtt{i}^2=\mathtt{j}^2=\mathtt{k}^2=-1,  \quad \mathtt{i}\mathtt{j}=-\mathtt{j}\mathtt{i}=\mathtt{k}, \quad \mathtt{j}\mathtt{k}=-\mathtt{k} \mathtt{j}= \mathtt{i},\quad \mathtt{k}\mathtt{i}=-\mathtt{i}\mathtt{k}=\mathtt{j}. $$
Any quaternion $\breve{q}$ can be simply written as $\breve{q}=\textit{Re}~ \breve{q}+ \textit{Im}~\breve{q}$ with real component $\textit{Re}~\breve{q}= q_0$ and imaginary component $\textit{Im}~\breve{q}= \mathtt{i}q_1+ \mathtt{j}q_2+\mathtt{k}q_3$. The quaternion conjugate $\bar{\breve{q}}$ and the modulus $|\breve{q}|$ of $\breve{q}$ are defined as 
$$\bar{\breve{q}}\doteq\textit{Re}~\breve{q} -\textit{Im}~\breve{q}=q_0-\mathtt{i} q_1-\mathtt{j}q_2-\mathtt{k}q_3,\quad  |\breve{q}|\doteq \sqrt{\breve{q}\bar{\breve{q}}}=\sqrt{q_0^2+q_1^2+q_2^2+q_3^2}.$$ 
%Consider quaternion matrix $\breve{\mathbf{Q}}=(\breve{q}_{ij})\in \mathbb{H}^{m\times n}$, denote its transpose $\breve{\mathbf{Q}}^T=(\breve{q}_{ji})\in \mathbb{H}^{n\times m}$ and its conjugate-transpose $\bar{\breve{\mathbf{Q}}}^T=(\bar{\breve{q}}_{ji})\in \mathbb{H}^{n\times m}$.

In the paper, we use $\mathcal{S}_l(\cdot)$, $l=0,1,2,3$ to extract the corresponding components of a quaternion number, vector or matrix, for instance, $\mathcal{S}_0(\breve{q})=q_0$,   $\mathcal{S}_3(\breve{q})=q_3$. 
For the sake of simplifying notations, we will use $\mathcal{S}_l(\cdot)$ to represent quaternions in the rest of paper. Precisely, quaternion number 
 $\breve{q}=\mathcal{S}_0(\breve{q})+\mathtt{i}\mathcal{S}_1(\breve{q})+\mathtt{j}\mathcal{S}_2(\breve{q})+\mathtt{k}\mathcal{S}_3(\breve{q})\in \mathbb{H}$, quaternion vector $\breve{\mathbf{q}}=\mathcal{S}_0(\breve{\mathbf{q}})+\mathtt{i}\mathcal{S}_1(\breve{\mathbf{q}})+\mathtt{j}\mathcal{S}_2(\breve{\mathbf{q}})+\mathtt{k}\mathcal{S}_3(\breve{\mathbf{q}})\in \mathbb{H}^m$, and quaternion matrix
$\breve{\mathbf{Q}}=\mathcal{S}_0(\breve{\mathbf{Q}})+\mathtt{i}\mathcal{S}_1(\breve{\mathbf{Q}})+\mathtt{j}\mathcal{S}_2(\breve{\mathbf{Q}})+\mathtt{k}\mathcal{S}_3(\breve{\mathbf{Q}})\in \mathbb{H}^{m\times n}$.

We define the inner product of two quaternion vectors as follows. Given $\breve{\mathbf{a}} \in \mathbb{H}^{m}$, $\breve{\mathbf{b}} \in \mathbb{H}^{m}$,  
Their inner product is defined as,
\begin{equation}\label{inner}
\langle \breve{\mathbf{a}},  \mathbf{\breve{b}}\rangle\doteq \sum^3_{l=0}\mathcal{S}^T_l(\breve{\mathbf{a}}) \mathcal{S}_l(\breve{\mathbf{b}}).
\end{equation} 
And the $l_2$ norm of any quaternion vector $\breve{\mathbf{q}} \in \mathbb{H}^{m}$ is given below,
\begin{equation}\label{norm2}
\|\breve{\mathbf{q}}\|_2 \doteq \langle\breve{\mathbf{q}}, \breve{\mathbf{q}}\rangle^{\frac{1}{2}}.
\end{equation} 
We use $\mat(\cdot)$ to represent all the components of a quaternion vector $\breve{\mathbf{q}}\in \mathbb{H}^m$ or matrix $\breve{\mathbf{Q}}\in \mathbb{H}^{m\times n}$, that is
\begin{eqnarray*}
\mat(\breve{\mathbf{q}})&=& \big(
           \begin{array}{cccc}      
\mathcal{S}_0(\breve{\mathbf{q}}) & \mathcal{S}_1(\breve{\mathbf{q}}) & \mathcal{S}_2(\breve{\mathbf{q}}) &\mathcal{S}_3(\breve{\mathbf{q}}) \end{array}
         \big)\in \mathbb{R}^{m\times 4}, \\
 \mat(\breve{\mathbf{Q}})&=&\big(
           \begin{array}{cccc}      
\mathcal{S}^T_0(\breve{\mathbf{Q}}) & \mathcal{S}^T_1(\breve{\mathbf{Q}}) & \mathcal{S}^T_2(\breve{\mathbf{Q}}) &\mathcal{S}^T_3(\breve{\mathbf{Q}})
 \end{array}
         \big)^T\in \mathbb{R}^{4m\times n}.
\end{eqnarray*} 
It is easy to verify that 
\begin{equation*}
\|~\breve{\mathbf{q}}~ \|_2=\|\mat(\breve{\mathbf{q}})\|_F
\end{equation*} 
 
In the following, we present an inequality based on the $l_2$ norm of quaternion vector which is useful in the paper.
\begin{lemma}\label{lem:q-norm}
For any quaternion matrix $\breve{\mathbf{Q}}\doteq [\breve{\mathbf{q}}_1,\cdots,\breve{\mathbf{q}}_n]\in \mathbb{H}^{m\times n}$, and any real vector $\mathbf{h}=[h_1,\cdots,h_n]^T\in \mathbb{R}^n$, 
$$
\|\breve{\mathbf{Q}}\mathbf{h}\|_2\leq \|\mathbf{h}\|_1\max_t\|\breve{\mathbf{q}}_t\|_2
$$
\end{lemma}
\begin{proof}
\begin{eqnarray*}
\|\breve{\mathbf{Q}}\mathbf{h}\|_2=\|\sum^n_{t=1}\breve{\mathbf{q}}_th_t\|_2 
=\|\sum^n_{t=1}h_t \mat(
          \breve{\mathbf{q}}_t)\|_F  \leq  \sum^n_{t=1}|h_t|\|\mat(\breve{\mathbf{q}}_t)\|_F = \sum^n_{t=1}|h_t| \|\breve{\mathbf{q}}_t\|_2\leq \|\mathbf{h}\|_1\max_t\|\breve{\mathbf{q}}_t\|_2.
\end{eqnarray*}
\end{proof}
\subsection{Stokes parameters}

Stokes parameters $S_0, S_1, S_2, S_3$ are used to describe the polarization state of electromagnetic radiation.  The first parameter $S_0\geq 0$ represents the total intensity of light, equals sum of the intensity of polarized and un-polarized light.  Parameters $S_1, S_2, S_3$ describe the polarization ellipse of light, specifically,  
 \begin{eqnarray}\label{stokes}
S_{1}&=&\phi S_0 \cos 2\psi \cos 2\chi, \nonumber \\
S_{2}&=&\phi S_0 \sin 2\psi \cos 2\chi, \\
S_{3}&=&\phi S_0 \sin 2\chi. \nonumber 
 \end{eqnarray}
$\phi S_0$ , $2\psi$  and $2\chi$ are the spherical coordinates of the three-dimensional vector of cartesian coordinates $(S_{1},S_{2},S_{3})$. $\phi$ is the degree of polarization, the metric of relative contribution of polarized part to the total intensity, 
\begin{equation*}
\phi=\frac{\sqrt{S^2_1+S^2_2+S^2_3}}{S_0}.
\end{equation*}
$\phi\in [0,1]$ implies the polarization state.  $\phi=1$ indicates fully polarized, and $\phi=0$  means un-polarized. For $\phi\in (0,1)$, we say it is partially polarized. 

These four Stokes parameters $(S_0,S_1,S_2,S_3)$ can be expressed by using quaternion algebra  as 
$$
\breve{q}=S_0+\mathtt{i}S_1+\mathtt{j}S_2+\mathtt{k}S_3\in \mathbb{H}.
$$
Due to the physical meaning indicated in \eqref{stokes}, the Stokes vector $(S_0,S_1,S_2,S_3)^T\in \mathbb{R}^{4}$ satisfies that 
\begin{equation}\label{S_cond}
S_0\geq 0 \quad and \quad S^2_1+S^2_2+S^2_3\leq S^2_0.
\end{equation}
Worth noting that, condition \eqref{S_cond} is equivalent to the positive semi-definiteness of a $2 \times 2$ Hermitian matrix $\mathbf{J}$,
$$
\mathbf{J}=\frac{1}{2}\left(
           \begin{array}{cc}
             S_0+S_2 & S_3+\mathtt{i}S_1  \\
            S_3-\mathtt{i}S_1   & S_0-S_2\\
           \end{array}
         \right)\in \mathbb{C}^{2\times 2}.
$$
More precisely, \eqref{S_cond} $\Leftrightarrow$ $\tr{(\mathbf{J})}\geq 0$ and $det{(\mathbf{J})} \geq 0$. The condition \eqref{S_cond} implies, for any quaternion $\breve{q}\in \mathbb{H}$ satisfies  $\textit{Re}~\breve{q}\geq 0$ and $|\textit{Im}~\breve{q}|^2 \leq (\textit{Re}~\breve{q})^2$. Define the set of constrained quaternions \cite{flamant2020quaternion} $\mathbb{H}_S\subset \mathbb{H}$ such that
\begin{equation}
\mathbb{H}_S \doteq \{\breve{q}\in \mathbb{H}|\textit{Re}~\breve{q} \geq 0 ~~ and ~~ |\textit{Im}~\breve{q}|^2 \leq (\textit{Re}~\breve{q})^2\}.
\end{equation}

For more details on polarization and Stokes parameters, we refer the interested reader to \cite{gil2017polarized} and the references therein.

\section{Definition and Properties}

In this section, we will introduce a linear mixing model for polarized matrix $\breve{\mathbf{M}}\in \mathbb{H}^{m\times n}_S$.  
It is known that each pixel of a polarization image contains light intensity and its polarization information, 
$$
\breve{m}_{ij}=\mathcal{S}_0(\breve{m}_{ij})+\mathtt{i}\mathcal{S}_1(\breve{m}_{ij})+\mathtt{j}\mathcal{S}_2(\breve{m}_{ij})+\mathtt{k}\mathcal{S}_3(\breve{m}_{ij})\in \mathbb{H}_{S}.
$$ 
To illustrate the motivation of the separable quaternion  model,  we will use spectro-polarimetric imaging as an example.  

As we know, in the application of hyperspectral unmixing, the linear mixing model for hyperspectral imaging (HSI) data assumes that the spectrum of a pixel is a linear weighted combination of the pure spectra (endmembers) of the components present in that pixel. One commonly used assumption in endmember extraction is the presence of pure pixels in the given images, i.e., for each distinct material, at least one pixel (i.e., spectral signature) exists in the image only contains that material. 
We note that if the imaginary components of all pixels are equal to zero, all lights are unpolarized. The  spectro-polarimetric image will degenerate to (HSI) data.  From the physical interpretation of light intensity and its polarization, spectro-polarimetric imaging is supposed to share similar material and structural properties as HSI. We consider pure pixel assumption in the linear mixing model for spectro-polarimetric data represented as quaternion form.

Accordingly, we will generalize the standard separable nonnegative model to the quaternion matrix factorization model. Formally, 
\begin{defn}\label{def:QSMF}
$\breve{\mathbf{M}}\in \mathbb{H}^{m\times n}_S$ is  separable if there exist $\breve{\mathbf{W}}\in \mathbb{H}^{m\times n}_S $  and nonnegative matrix $\mathbf{H} \in \mathbb{R}^{r\times n}_+$  such that
\begin{equation}\label{sep}
\breve{\mathbf{M}}= \breve{\mathbf{W}}\mathbf{H},
\end{equation}
and each column of $\breve{\mathbf{W}}$ is a column of $\breve{\mathbf{M}}$. Precisely,  $\breve{\mathbf{W}}=\breve{\mathbf{M}}(:,\mathcal{K})$,
where  $\mathcal{K} $ is column set of cardinality $r$ of $\breve{\mathbf{M}}$.
\end{defn}
For simplicity, we call a quaternion matrix $\breve{\mathbf{M}}$ the separable quaternion matrix (SQ-matrix) if it has the form of \eqref{sep}.  Note that the aim of the  separable quaternion factorization is to find the minimal number of source columns $\breve{\mathbf{W}}$ from the given $\breve{\mathbf{M}}$,  $r$ here refers to the minimal number of columns in $\breve{\mathbf{W}}$.

As we know that every entry of the quaternion matrix $\breve{\mathbf{M}}$ is in the constrained quaternion set $\mathbb{H}_S$, it is necessary to investigate the plausibility of the assumption in SQMF model \eqref{sep}. The result is shown in the following proposition. 
\begin{proposition}\label{Prop:1}
If $ \breve{\mathbf{W}}\in \mathbb{H}^{m\times r}_S$, $\mathbf{H} \in \mathbb{R}^{r\times n}_+$, $\breve{\mathbf{M}}=\breve{\mathbf{W}}\mathbf{H}$, then $\breve{\mathbf{M}}\in \mathbb{H}^{m\times n}_S$.
\end{proposition}
\begin{proof}
Let $\breve{m}_{ij}=\mathcal{S}_0(\breve{m}_{ij})+\mathtt{i}\mathcal{S}_1(\breve{m}_{ij})+\mathtt{j}\mathcal{S}_2(\breve{m}_{ij})+\mathtt{k}\mathcal{S}_3(\breve{m}_{ij})$, $\breve{w}_{ij}=\mathcal{S}_0(\breve{w}_{ij})+\mathtt{i}\mathcal{S}_1(\breve{w}_{ij})+\mathtt{j}\mathcal{S}_2(\breve{w}_{ij})+\mathtt{k}\mathcal{S}_3(\breve{w}_{ij})$.
From $\breve{m}_{ij}=\sum^r_{t=1}\breve{w}_{it}h_{tj}$, we have
$$
\mathcal{S}_0(\breve{m}_{ij})=\sum\limits^r_{t=1}\mathcal{S}_0(\breve{w}_{it})h_{tj}, \quad 
\mathcal{S}_l(\breve{m}_{ij})=\sum\limits^r_{t=1}\mathcal{S}_l(\breve{w}_{it})h_{tj}, \quad l=1,2,3.
$$
Easy to know that $\mathcal{S}_0(\breve{m}_{ij}) \geq 0 $. To prove $\breve{m}_{ij}\in \mathbb{H}_S$, in the following, we will show that $\sum\limits^3_{l=1}\mathcal{S}_l^2(\breve{m}_{ij})\leq \mathcal{S}_0^2(\breve{m}_{ij})$.
As we know, 
\begin{eqnarray*}
\sum^3_{l=1}\mathcal{S}^2_l(\breve{m}_{ij})&=&\sum^3_{l=1}\big(\mathcal{S}_l(\breve{w}_{i1})h_{1j}+\mathcal{S}_l(\breve{w}_{i2})h_{2j}+\cdots+\mathcal{S}_l(\breve{w}_{ir})h_{rj}\big)^2\\
&=& \sum^r_{t=1}h^2_{tj}\big(\sum^3_{l=1}\mathcal{S}^2_l(\breve{w}_{it})\big)+2\sum^r_{t_1<t_2}h_{t_1j}h_{t_2j}\big(\sum^3_{l=1} \mathcal{S}_l(w_{it_1}) \mathcal{S}_l(\breve{w}_{it_2})\big)\\
&\leq & \sum^r_{t=1}h^2_{tj}\mathcal{S}^2_0(\breve{w}_{it})+2\sum^r_{t_1<t_2}h_{t_1j}h_{t_2j}\big(\sum^3_{l=1} \mathcal{S}^2_l(\breve{w}_{it_1}) \big)^{\frac{1}{2}}\big(\sum^3_{l=1}\mathcal{S}^2_l(\breve{w}_{it_2}) \big)^{\frac{1}{2}}\\
&\leq & \sum^r_{t=1}h^2_{tj}\mathcal{S}^2_0(\breve{w}_{it})+2\sum^r_{t_1<t_2}h_{t_1j}h_{t_2j}\mathcal{S}_0(\breve{w}_{it_1}) \mathcal{S}_0(\breve{w}_{it_2})\\
&=& \big(\sum^r_{t=1}h_{tj}\mathcal{S}_0(\breve{w}_{it})\big)^2=\mathcal{S}^2_0(\breve{m}_{ij}).
\end{eqnarray*}
The results follow.
\end{proof}
Proposition \ref{Prop:1} implies that any nonnegative weighted linear combination of quaternions in $\mathbb{H}_S$ also lies in $\mathbb{H}_S$, i.e., $\mathbb{H}_S$ is convex cone. In the following proposition, we will show the connection between a $r$-separable quaternion matrix $\breve{\mathbf{M}}$ and its contained four components.  
 
\begin{property}\label{Prop:equi}
$\breve{\mathbf{M}}\in \mathbb{H}^{m\times n}_S$ is quaternion separable matrix, $\breve{\mathbf{M}}=\breve{\mathbf{M}}(:,\mathcal{K})\mathbf{H}$ if and only if $\mathcal{S}_l(\breve{\mathbf{M}})$ are separable matrices for all $l=0,1,2,3$, that is, there exists $\mathbf{H}_l\in \mathbb{R}^{r_l\times n}_+$, such that
$$\mathcal{S}_l(\breve{\mathbf{M}})=\mathcal{S}_l(\breve{\mathbf{M}})(:,\mathcal{K}_l)\mathbf{H}_l,\quad |\mathcal{K}_l|=r_l\leq r,\quad $$
and at least one column set $\mathcal{K}_l$ contains the others, i.e., for some $p\in \{0,1,2,3\},$
 $$
\bigcup^3_{l=0}\mathcal{K}_{l}= \mathcal{K}_p\doteq
\mathcal{K},
$$
and there exists matrix $\mathbf{Q}\in \mathbb{R}^{r\times (n-r)}_+$, such that
$$
\mathbf{H}_l(:,\bar{\mathcal{K}})=[\mathbf{I}_{r_l},\mathbf{H}_l(:,\mathcal{K}-\mathcal{K}_l)]\Pi \mathbf{Q},
 $$
where $\Pi$ is appropriate permutation matrix, $\bar{\mathcal{K}} \doteq \{1,\cdots,n\}-\mathcal{K}$, and $\mathbf{I}_{r_l}$ is identity matrix of $r_l$. The factor matrices $(\breve{ \mathbf{W}},\mathbf{H}):$
$$
\breve{\mathbf{W}}:~~ \breve{\mathbf{W}}=\breve{\mathbf{M}}(:,\mathcal{K}); \quad \mathbf{H}:~~ \mathbf{H}(:\mathcal{K})=\mathbf{I}_{r_l}, \mathbf{H}(:,\bar{\mathcal{K}})=\mathbf{Q}.
$$
\end{property}
\begin{proof}
$(\Longrightarrow)$ If $\breve{\mathbf{M}}$ is quaternion separable matrix, the results are obviously from Definition  \ref{def:QSMF}.

$(\Longleftarrow)$ Without loss of generality, assume $\mathcal{K}_3=\bigcup\limits^3_{l=0}\mathcal{K}_l=\{1,2,\cdots, r\}\doteq\mathcal{K}$, and $\mathcal{K}_0=\{1,\cdots, r_0\}\subset \mathcal{K}$, with $r_0\leq r$. From 
 $\mathcal{S}_0(\breve{\mathbf{M}})=\mathcal{S}_0(\breve{\mathbf{M}})(:,\mathcal{K}_0)\mathbf{H}_0,$ let $\mathbf{A}_0=\mathcal{S}_0(\breve{\mathbf{M}})(:,\mathcal{K}_0)$, we deduce that
 $$
 \mathcal{S}_0(\breve{\mathbf{M}})=\left(\begin{array}{ccc}
\mathbf{A}_0&\mathbf{A}_0\mathbf{H}^1_0&\mathbf{A}_0\mathbf{H}^2_0\end{array}
            \right),           
$$
where 
$ 
\mathbf{H}^1_0=\mathbf{H}_0(:,\mathcal{K}_3-\mathcal{K}_0)\in \mathbb{R}^{r_0\times (r-r_0)}
 $ and $ 
\mathbf{H}^2_0=\mathbf{H}_0(:,\bar{\mathcal{K}}_3)\in \mathbb{R}^{r_0\times (n-r)}
 $. 
 From $\mathbf{H}^2_0=[\mathbf{I}_{r_0},\mathbf{H}^1_0]\mathbf{Q}$, we have
 $$\mathbf{A}_0\mathbf{H}^2_0=\mathbf{A}_0\mathbf{Q}_1+\mathbf{A}_0\mathbf{H}^1_0\mathbf{Q}_2, \quad \text{where}\quad \mathbf{Q}_1=\mathbf{Q}(\mathcal{K}_0,:)~~ \text{and} ~~ \mathbf{Q}_2=\mathbf{Q}(\mathcal{K}_3-\mathcal{K}_0,:) $$ 
 
 Then
\begin{eqnarray*}
 \mathcal{S}_0(\breve{\mathbf{M}})&=&\left(\begin{array}{ccc}
\mathbf{A}_0&\mathbf{A}_0\mathbf{H}^1_0&\mathbf{A}_0\mathbf{Q}_1+\mathbf{A}_0\mathbf{H}^1_0\mathbf{Q}_2\end{array}
            \right)\\
&=& 
 \left(\begin{array}{cc}
\mathbf{A}_0&\mathbf{A}_0\mathbf{H}^1_0\end{array}
            \right)\left(\begin{array}{ccc}
\mathbf{I}_{r_0}&\mathbf{0}&\mathbf{Q}_1 \\
\mathbf{0}&\mathbf{I}_{r-r_0}&\mathbf{Q}_2 \\
\end{array}
            \right). 
\end{eqnarray*}
\text{Denote} $$ \quad \mathcal{S}_0(\breve{\mathbf{W}})= \left(\begin{array}{cc}
\mathbf{A}_0&\mathbf{A}_0\mathbf{H}^1_0\end{array}
            \right),\quad
\mathbf{H}= \left(\begin{array}{ccc}
\mathbf{I}_{r_0}&\mathbf{0}&\mathbf{Q}_1 \\
\mathbf{0}&\mathbf{I}_{r-r_0}&\mathbf{Q}_2 \\
\end{array}
            \right),       
             $$
we have $$\mathcal{S}_0(\breve{\mathbf{M}})=\mathcal{S}_0(\breve{\mathbf{W}})\mathbf{H}=\mathcal{S}_0(\breve{\mathbf{M}})(:,\mathcal{K})\mathbf{H}.$$ 
Similarly, 
$$\mathcal{S}_1(\breve{\mathbf{M}})=\mathcal{S}_1(\breve{\mathbf{W}})\mathbf{H}=\mathcal{S}_1(\breve{\mathbf{M}})(:,\mathcal{K})\mathbf{H},\quad \mathcal{S}_2(\breve{\mathbf{M}})=\mathcal{S}_2(\breve{\mathbf{W}})\mathbf{H}=\mathcal{S}_2(\breve{\mathbf{M}})(:,\mathcal{K})\mathbf{H}.$$
Note that 
$$\mathcal{S}_3(\breve{\mathbf{M}})=\mathcal{S}_3(\breve{\mathbf{W}})\mathbf{H}=\mathcal{S}_3(\breve{\mathbf{M}})(:,\mathcal{K})\mathbf{H}.$$ 
We then construct $\breve{\mathbf{W}}\in \mathbb{H}^{m\times r}_S$, and $\mathbf{H}\in \mathbb{R}^{r\times n}_+$, such that
$\breve{\mathbf{M}}=\breve{\mathbf{W}}\mathbf{H}$. The results follow.  
\end{proof}
From Property \ref{Prop:equi}, we remark that quaternion matrix $\breve{\mathbf{M}}$ can not be guaranteed to be $r$-separable, even if its all four components are $r$-separable. 
 
The following property shows that quaternion separable matrix factorization admits a unique solution up to permutation and scaling.
\begin{property}[Uniqueness]\label{prop:uniqueness}
$\breve{\mathbf{M}}\in \mathbb{H}^{m\times n}_S$ is $r$-quaternion separable matrix, that is, $\breve{\mathbf{M}}=\breve{\mathbf{W}}\mathbf{H}$, $\breve{\mathbf{W}}=\breve{\mathbf{M}}(:,\mathcal{K})\in \mathbb{H}^{m\times r}_S$, $\mathbf{H}\in \mathbb{R}^{r\times n}_+$. Then $\breve{\mathbf{W}}$ is unique up to permutation and scaling. And if $\mat(\breve{\mathbf{M}})=r$, $\mathbf{H}$ is also unique up to permutation and scaling .
\end{property}
\begin{proof}

$\breve{\mathbf{M}}$ is $r$ quaternion separable matrix, then let  
\begin{equation}\label{XY}
\mathbf{X}= \mat(\breve{\mathbf{M}}) \in \mathbb{R}^{4m\times n},\quad 
\mathbf{Y}=  \mat(\breve{\mathbf{W}}) \in \mathbb{R}^{4m\times r} 
\end{equation} 
From Proposition \ref{Prop:equi},  $\mathbf{X}=\mathbf{Y}\mathbf{H}$ and $\mathbf{Y}=\mathbf{X}(:,\mathcal{K})$. 
If $\mathbf{Y}$ is not unique up to permutation and scaling, there exist $(\mathbf{Y}_1, \mathbf{H}_1)$ and $(\mathbf{Y}_2,\mathbf{H}_2)$ such that
\begin{equation}\label{YH1}
\mathbf{X}=\mathbf{Y}_1\mathbf{H}_1,~~ \mathbf{Y}_1=\mathbf{X}(:,\mathcal{K}),~~\mathbf{H}_1\in \mathbb{R}^{r\times n}_+;
\end{equation}
\begin{equation}\label{YH2}
\mathbf{X}=\mathbf{Y}_2\mathbf{H}_2,~~\mathbf{Y}_2=\mathbf{X}(:,\tilde{\mathcal{K}}),~~\mathbf{H}_2\in \mathbb{R}^{r\times n}_+.
\end{equation}
From \eqref{YH1}, we deduce that $\mathbf{Y}_2=\mathbf{Y}_1\mathbf{V}_1$,  $\mathbf{V}_1\in \mathbb{R}^{r\times r}_+$, and the columns of $\mathbf{V}_1$ are from $\mathbf{H}_1$. Similarly, from \eqref{YH2}, $\mathbf{Y}_1=\mathbf{Y}_2\mathbf{V}_2$, $\mathbf{V}_2\in \mathbb{R}^{r\times r}_+$,  and the columns of $\mathbf{V}_2$ are from $\mathbf{H}_2$. Hence we have,
\begin{equation}
\mathbf{Y}_1(\mathbf{I}_r-\mathbf{V}_1\mathbf{V}_2)=0,\quad \mathbf{Y}_2(\mathbf{I}_r-\mathbf{V}_2\mathbf{V}_1)=0.
\end{equation}
Let $\mathbf{G}=\mathbf{V}_1\mathbf{V}_2$, and assume that $\mathbf{I}_r-\mathbf{G}\neq 0$. Without loss of generality, let the $j$-th column of $\mathbf{G}$ is non zero column, then 
\begin{equation}\label{prop3:eq}
(1-g_{jj})\mathbf{Y}_1(:,j)-\sum^r_{i\neq j}g_{ij}\mathbf{Y}_1(:,i)=0.
\end{equation}
Note that $g_{ij}\neq 0$ for some $i$, and $\mathbf{Y}_1(:,i)\neq 0$ for all $i$, 

If $1-g_{jj}\neq 0$, then 
$$\mathbf{Y}_1(:,j)=\frac{\sum^r_{i\neq j}g_{ij}\mathbf{Y}_1(:,i)}{1-g_{jj}}.$$
Notice that $\mathbf{Y}_1(1:m,:)=\mathcal{S}_0(\breve{\mathbf{W}})\geq 0$, and $\mathbf{G}\geq 0$, easy to get 
$1-g_{jj}>0$.  It implies that $\mathbf{Y}_1(:,j)$ can be represented as linear combination of other columns of $\mathbf{Y}_1$, with nonnegative coefficients. This is  contradict  to minimal $r$-separability of $\breve{\mathbf{M}}$. Hence, $g_{jj}=1$.

From \eqref{prop3:eq}, we get
\begin{equation*} 
\sum^r_{i\neq j}g_{ij}\mathbf{Y}_1(:,i)=0.
\end{equation*}
From the fact of $\mathbf{Y}_1(1:m,:)=\mathcal{S}_0(\breve{\mathbf{W}})\geq 0$, and $\mathbf{G}\geq 0$, we obtain  $g_{ij}=0$ for $i\neq j$.  

Hence $\mathbf{G}=\mathbf{I}_r$,  that is $
\mathbf{V}_1\mathbf{V}_2=\mathbf{I}_r$. Similarly, $\mathbf{V}_2\mathbf{V}_1=\mathbf{I}_r.$  As $\mathbf{V}_1$ and $\mathbf{V}_2$ are nonnegative matrices, $\mathbf{V}_1$ and $\mathbf{V}_2$ can only be permutation matrix or scaling matrix. Therefore, $\mathbf{Y}$ is unique up to permutation and scaling. 

Specifically, if $\rank(\mat(\breve{\mathbf{M}}))=r$, that is $rank(\mathbf{X})=r$, one can easily deduce that  $\mathbf{Y}$ has full column rank, or else it will contradict to assumption of $rank(\mathbf{X})=r$. Since $\mathbf{X}=\mathbf{Y}\mathbf{H}$, we get that $\mathbf{H}$ is unique up to permutation and scaling from fundamental results of matrix equation.
\end{proof}
\begin{remark}
Reference \cite{flamant2020quaternion} indicates that if $\mathcal{S}_0(\breve{\mathbf{M}})$ is $r$-separable,   $\breve{\mathbf{W}}$ is unique up to permutation and scaling. 
Property \ref{prop:uniqueness} extends to
the case that $\mathcal{S}_0(\breve{\mathbf{M}})$ is $r_1$-separable when $r_1<r$.
\end{remark}
We remark that for a quaternion $r$ separable  matrix $\breve{\mathbf{M}}\in \mathbb{H}^{m\times n}_S$, in general, $\mathbf{H}$ is not unique up to permutations and scaling when $r>rank(\mat(\breve{\mathbf{M}}))$. In the following, we will give an example to show this fact. 
\begin{example}
$$
\breve{\mathbf{M}}=\mathcal{S}_0(\breve{\mathbf{M}})+\mathtt{i}\mathcal{S}_1(\breve{\mathbf{M}})+\mathtt{j}\mathcal{S}_2(\breve{\mathbf{M}})+\mathtt{k}\mathcal{S}_3(\breve{\mathbf{M}})
$$
$$
\mathcal{S}_0(\breve{\mathbf{M}})=\left(\begin{array}{ccccc}
1&1&1&0.5&0.9\\
0.5&1&1&0.75&0.85\\
0.5&0.5&1&0.5&0.65
\end{array}
            \right),\quad             
\mathcal{S}_1(\breve{\mathbf{M}})=\left(\begin{array}{ccccc}
0.45&0.4&0.1&0.025&0.245\\
-0.05&0.3&0.4&0.375&0.275\\
0.15&0.15&0.25&0.125&0.175
\end{array}
\right),     
$$   
 $$
 \mathcal{S}_2(\breve{\mathbf{M}})=\frac{1}{2}\left(\begin{array}{ccccc}
-0.1&0.2&0.3&0.3&0.19\\
0.04&-0.5&-0.6&-0.57&-0.436\\
0.07&0.08&0.9&0.455&0.399
\end{array}
            \right),~          
\mathcal{S}_3(\breve{\mathbf{M}})=\frac{1}{2}\left(\begin{array}{ccccc}
0.1&-0.4&0.7&0.1&0.13\\
-0.02&0.5&0.8&0.66&0.518\\
0.03&0.06&0.9&0.465&0.387
\end{array}
            \right). 
 $$
$\breve{\mathbf{M}}\in \mathbb{H}^{3\times 5}_S$ is a 4-separable quaternion matrix, and $\rank(\mat(\breve{\mathbf{M}}))=3$. We can get that $\breve{\mathbf{M}}=\breve{\mathbf{W}}\mathbf{H}$, with
$
\breve{\mathbf{W}}=\breve{\mathbf{M}}(:,1:4)
$. But $\mathbf{H}$ can be
$$
\mathbf{H}=\left(\begin{array}{ccccc}
1&0&0&0&0.1\\
0&1&0&0&0.4\\
0&0&1&0&0.4\\
0&0&0&1&0\\
\end{array}
\right),~~\text{or}~~
\mathbf{H}=\left(\begin{array}{ccccc}
1&0&0&0&0.3\\
0&1&0&0&0.2\\
0&0&1&0&0.2\\
0&0&0&1&0.4\\
\end{array}
\right).
$$
\end{example}

\section{Algorithms} 

This section proposes quaternion successive projection algorithm extended from SPA, to identify the column set $\mathcal{K}$ from a given matrix $\breve{\mathbf{M}}$. The extracted columns will form the source matrix $\breve{\mathbf{W}}$. For solving activation matrix $\mathbf{H}$, a quaternion hierarchical nonnegative least squares method is proposed. 

\subsection{Quaternion SPA}

Inspired from a well-known separable NMF algorithm named successive projection algorithm (SPA), we design a fast algorithm specifically for  separable quaternion matrix factorization. SPA was introduced in~\cite{araujo2001successive} regarding spectral unmixing and has been well discussed in~\cite{ma2014signal, gillis2014and}.
Moreover, it is robust in the presence of noise~\cite{gillis2014fast}. 
SPA assumes that the input matrix $\mathbf{M}\in \mathbb{R}^{m\times n}$ has similar form as \eqref{def:QSMF}, that is
$\mathbf{M}= \mathbf{M}(:,\mathcal{K}) [\mathbf{I}_r, \mathbf{H}'] \Pi$,
where $\Pi$ is a permutation matrix and $\mathbf{H}' \geq 0$ and $||\mathbf{H}'(:,t)||_1 \leq 1$ for all~$t$. The main idea behind SPA which sequentially identifies the columns in $\mathcal{K}$ is given as follows:
at each step, it first determines  the column of $\mathbf{M}$ that has the largest $\ell_2$ norm,  and then projects all columns of $\mathbf{M}$ onto the orthogonal complement of the extracted column.

To start successive projection algorithm on quaternion field, given a quaternion separable matrix $\breve{\mathbf{M}}\in \mathbb{H}^{m\times n}_S$, there exists an index set $\mathcal{K}$ of cardinality $r$, and some $\mathbf{H}\in \mathbb{R}^{r\times n}_+$, such that 
$\breve{\mathbf{M}} = \breve{\mathbf{W}}\mathbf{H}$, where $\breve{\mathbf{W}}=\breve{\mathbf{M}}(:\mathcal{K})$, $\mathbf{H}=[\mathbf{I}_r,\mathbf{U}]\Pi$, with $\mathbf{U}=[\mathbf{u}_1,\cdots,\mathbf{u}_t]\in \mathbb{R}^{r\times t}$ and $t+r=n$. 

Note that the assumption of $||\mathbf{H}'(:,t)||_1 \leq 1$ for all~$t$ is necessary to guarantee SPA identifies the correct column set $\mathcal{K}$. We remark that it also plays a critical role in the proposed quaternion SPA. From Lemma \ref{lem:q-norm}, one can easily verify that the column of largest $\ell_2$ norm of quaternion separable matrix $\breve{\mathbf{M}}$ will be identified from $\mathcal{K}$ if $||\mathbf{u}_t||_1 \leq 1$ for all~$t$. However in general,  $||\mathbf{u}_t||_1  \leq 1$ is not  satisfied.  To address this issue, a normalize operator can be applied on each column of SQ-matrix $\breve{\mathbf{M}}$. How to define a suitable normalize operator for a quaternion matrix will be the key to propose quaternion SPA. 

Suppose that $f: \mathbb{H}^n_S\rightarrow \mathbb{R}_+$ is a suitable normalization function for any $\breve{\mathbf{q}}\in \mathbb{H}^n_S$,    and $\dfrac{\breve{\mathbf{q}}}{f(\breve{\mathbf{q}})}$ is the normalized vector of $\breve{\mathbf{q}}$. Apply the normalize operator on SQ-matrix $\breve{\mathbf{M}}$,  for any $t$,
$$
\dfrac{\breve{\mathbf{W}}\mathbf{u}_t}{f(\breve{\mathbf{W}}\mathbf{u}_t)}=\dfrac{\sum\limits^r_{i=1}\breve{\mathbf{w}}_iu_{it}}{f(\sum^r_{i=1}\breve{\mathbf{w}}_iu_{it})}=\sum^r_{i=1}\dfrac{\breve{\mathbf{w}}_i}{f(\breve{\mathbf{w}}_i)}\dfrac{u_{it}f(\breve{\mathbf{w}}_i)}{f(\sum^r_{i=1}\breve{\mathbf{w}}_iu_{it})}\doteq \sum^r_{i=1}\dfrac{\breve{\mathbf{w}}_i}{f(\breve{\mathbf{w}}_i)} \tilde{u}_{it}.
$$
where $\tilde{u}_{it}\doteq \dfrac{u_{it}f(\breve{\mathbf{w}}_i)}{f(\sum^r_{i=1}\breve{\mathbf{w}}_iu_{it})}$. Let
$\tilde{\mathbf{u}}_t\doteq (\tilde{u}_{it})$, to guarantee $||\tilde{\mathbf{u}}_t||_1  \leq 1$, that is
$ 
\sum\limits^r_{i=1}\dfrac{u_{it}f(\breve{\mathbf{w}}_i)}{f(\sum\limits^r_{i=1}\breve{\mathbf{w}}_iu_{it})}\leq 1.
$ We have, 
\begin{equation}\label{spa:cond}
\sum\limits^r_{i=1}u_{it}f(\breve{\mathbf{w}}_i)\leq f(\sum\limits^r_{i=1}\breve{\mathbf{w}}_iu_{it}).
\end{equation}
Hence,  we define the normalization function  $f(\breve{\mathbf{q}})\doteq\|\mathcal{S}_0(\breve{\mathbf{q}})\|_1$ that can satisfy the condition \eqref{spa:cond}. The quaternion successive projection algorithm is presented as follows.
\begin{algorithm}[H] 
\caption{Quaternion Successive Projection Algorithm \label{algo:alfgm}}
\begin{algorithmic}[1]
\REQUIRE Near-separable quaternion matrix $\breve{\mathbf{M}}=\breve{\mathbf{W}}[I_r,H]\Pi+\breve{\mathbf{N}}\in \mathbb{H}^{m\times n}_S$, where $\Pi$ is a permutation and $\breve{\mathbf{N}}$ is the noise.

\ENSURE Set of $r$ indices $\mathcal{K}$  such that $\breve{\mathbf{M}}(:,\mathcal{K})\approx \breve{\mathbf{W}}$.
 
\STATE  \textbf{Normalization:} $\breve{\mathbf{M}}=\left(       \begin{array}{cccc}\frac{\breve{\mathbf{m}}_1}{\|\mathcal{S}_0(\breve{\mathbf{m}}_1)\|_1},\frac{\breve{\mathbf{m}}_2}{\|\mathcal{S}_0(\breve{\mathbf{m}}_2)\|_1},\cdots,\frac{\breve{\mathbf{m}}_n}{\|\mathcal{S}_0(\breve{\mathbf{m}}_n)\|_1}
\end{array}
            \right). 
$
\STATE   Let $\mathbf{Z}=\breve{\mathbf{M}}$, $\mathcal{K}=\{ \}$.
%$X\leftarrow 0_{n,n}$;
%$Y\leftarrow 0_{m,m}$;

\FOR{$i$ = 1 : r}

\STATE $k=\mathop{\arg\max}_j\|\breve{\mathbf{z}}_j\|_{2}$.

\STATE Compute inner product vector $\mathbf{p}$\\
 $\mathbf{p}\doteq \left( \begin{array}{cccc}
 <\breve{\mathbf{z}}_k, \breve{\mathbf{z}}_1,> &<\breve{\mathbf{z}}_k, \breve{\mathbf{z}}_2,> & \cdots & <\breve{\mathbf{z}}_k, \breve{\mathbf{z}}_n>
 \end{array}
            \right)^T$;

\STATE \textbf{Projection:} $\breve{\mathbf{Z}}=\breve{\mathbf{Z}}-\frac{\breve{\mathbf{z}}_k\mathbf{p}^T}{\|\breve{\mathbf{z}}_k\|^2_{2}} $;

\STATE $\mathcal{K}=\mathcal{K}\cup \{k\}$.

\ENDFOR
\end{algorithmic}
\end{algorithm}

\begin{remark}
Quaternion SPA continues the advantages of the classic SPA, which are listed in the following.
\begin{itemize}
\item Quaternion SPA runs very fast, it can only be implemented in $\mathcal{O}(mnr)$.

\item No parameters need to be tuned.
\end{itemize}
\end{remark}
The correctness of quaternion SPA in the noiseless case is proven in the following theorem in details. 
 \begin{theorem}
Given  $r$-separable quaternion matrix $\breve{\mathbf{M}}\in \mathbb{H}^{m\times n}_S$,  quaternion SPA identifies the column set $\mathcal{K}$ correctly if $\rank(\mat(\breve{\mathbf{M}}))=r$. 
 \end{theorem}
\begin{proof}
Without loss of generality, we let the  column set $\mathcal{K}=\{1,2,\cdots r\}$,
$$
\breve{\mathbf{M}}=\left(\begin{array}{cccccccc}
\breve{\mathbf{w}}_1&\breve{\mathbf{w}}_2&\cdots&\breve{\mathbf{w}}_r&|& \breve{\mathbf{W}}\mathbf{u}_1&\cdots&\breve{\mathbf{W}}\mathbf{u}_t
\end{array}
            \right),
$$
where $\breve{\mathbf{W}}=\left(\begin{array}{cccc}
\breve{\mathbf{w}}_1&\breve{\mathbf{w}}_2&\cdots&\breve{\mathbf{w}}_r\end{array}
            \right)$.  After the normalization, 
$$
 \breve{\mathbf{M}}'=\left(\begin{array}{cccccccc}
\frac{\breve{\mathbf{w}}_1}{\|\mathcal{S}_0(\breve{\mathbf{w}}_1)\|_1}&\frac{\breve{\mathbf{w}}_2}{\|\mathcal{S}_0(\breve{\mathbf{w}}_2)\|_1}&\cdots&\frac{\breve{\mathbf{w}}_r}{\|\mathcal{S}_0(\breve{\mathbf{w}}_r)\|_1}&|& \frac{\breve{\mathbf{W}}\mathbf{u}_1}{\|\mathcal{S}_0(\breve{\mathbf{W}}\mathbf{u}_1)\|_1}&\cdots&\frac{\breve{\mathbf{W}}\mathbf{u}_t}{\|\mathcal{S}_0(\breve{\mathbf{W}}\mathbf{u}_t)\|_1}
\end{array}
            \right).
$$
Let $ \breve{\mathbf{w}}'_i\doteq \frac{\breve{\mathbf{w}}_i}{\|\mathcal{S}_0(\breve{\mathbf{w}}_i)\|_1}$,  $i\in \{1,2,\cdots r\}$, and $\breve{\mathbf{W}'}=\left(\begin{array}{cccc}
\breve{\mathbf{w}}'_1&\breve{\mathbf{w}}'_2&\cdots&\breve{\mathbf{w}}'_r\end{array}
            \right)$.  For $j\in \{1,2,\cdots t\}$,
$$
\frac{\breve{\mathbf{W}}\mathbf{u}_j}{\|\mathcal{S}_0(\breve{\mathbf{W}}\mathbf{u}_j)\|_1}=\frac{\sum\limits^r_{i=1}\breve{\mathbf{w}}_iu_{ij}}{\|\mathcal{S}_0(\breve{\mathbf{W}}\mathbf{u}_j)\|_1}=\sum\limits^r_{i=1} \breve{\mathbf{w}}'_i\frac{u_{ij}\|\mathcal{S}_0(\breve{\mathbf{w}}_i)\|_1}{\|\mathcal{S}_0(\breve{\mathbf{W}}\mathbf{u}_j)\|_1}.
$$
Let $\bar{u}_{ij}=\frac{u_{ij}\|\mathcal{S}_0(\breve{\mathbf{w}}_i)\|_1}{\|\mathcal{S}_0(\breve{\mathbf{W}}\mathbf{u}_j)\|_1}$, then
$ \breve{\mathbf{M}}'=\left(\begin{array}{cccccccc}
 \breve{\mathbf{w}}'_1& \breve{\mathbf{w}}'_2&\cdots& \breve{\mathbf{w}}'_r&|& \breve{\mathbf{W}}'\bar{\mathbf{u}}_1&\cdots& \breve{\mathbf{W}}'\bar{\mathbf{u}}_t
\end{array}
            \right).$ In the following, we will show that the first column will be chosen from $ \breve{\mathbf{W}}'$.  
            
Note $\mathcal{S}_0(\breve{\mathbf{W}})\geq 0$ and $\mathbf{U}\geq 0$, for $j\in \{1,2,\cdots t\}$, 
$$
\|\mathcal{S}_0(\breve{\mathbf{W}}\mathbf{u}_j)\|_1=\|\mathcal{S}_0(\sum^r_{i=1}\breve{\mathbf{w}}_iu_{ij})\|_1=\sum^r_{i=1}u_{ij}\|\mathcal{S}_0(\breve{\mathbf{w}}_i)\|_1.
$$           
Then we have 
$$
\sum^r_{i=1}\bar{u}_{ij}=\frac{\sum\limits^r_{i=1}u_{ij}\|\mathcal{S}_0(\breve{\mathbf{w}}_i)\|_1}{\sum\limits^r_{i=1}u_{ij}\|\mathcal{S}_0(\breve{\mathbf{w}}_i)\|_1}=1.
$$
Now 
\begin{eqnarray}\label{Them1:eq1}
\|\breve{\mathbf{W}}'\mathbf{\bar{u}}_{j}\|^2_2=\|\sum^r_{i=1} \breve{\mathbf{w}}'_i\bar{u}_{ij}\|^2_2\leq \sum^r_{i=1}\| \breve{\mathbf{w}}'_i\|^2_2 \bar{u}_{ij}^2 \leq \max_i\| \breve{\mathbf{w}}'_i\|^2_2 \sum^r_{i=1}\bar{u}_{ij}^2 \leq \max_i\| \breve{\mathbf{w}}'_i\|^2_2 (\sum^r_{i=1}\bar{u}_{ij})^2 =\max_i\|\breve{\mathbf{w}}'_i\|^2_2
\end{eqnarray}
Hence the first column will be selected from $ \breve{\mathbf{W}}'$.

Without loss  of generality, assume that $\breve{\mathbf{w}}_1$ is chosen, then in the projection procedure,
\begin{eqnarray*}
\mathbf{Z}&=&\breve{\mathbf{M}}'-\left(\begin{array}{ccccccc}
\frac{< \breve{\mathbf{w}}'_1, \breve{\mathbf{w}}'_1> \breve{\mathbf{w}}'_1 }{< \breve{\mathbf{w}}'_1, \breve{\mathbf{w}}'_1>}& \cdots & \frac{ < \breve{\mathbf{w}}'_1, \breve{\mathbf{w}}'_r> \breve{\mathbf{w}}'_1}{< \breve{\mathbf{w}}'_1, \breve{\mathbf{w}}'_1>} &|& 
\frac{< \breve{\mathbf{w}}'_1, \breve{\mathbf{W}}'\bar{\mathbf{u}}_1> \breve{\mathbf{w}}'_1}{< \breve{\mathbf{w}}'_1, \breve{\mathbf{w}}'_1>}&\cdots &\frac{< \breve{\mathbf{w}}'_1, \breve{\mathbf{W}}'\bar{\mathbf{u}}_t> \breve{\mathbf{w}}'_1}{< \breve{\mathbf{w}}'_1, \breve{\mathbf{w}}'_1>}
\end{array}
            \right)\\
&=& \left(\begin{array}{cccccccc}
\mathbf{0}&|&  \breve{\mathbf{w}}'_2-\frac{ < \breve{\mathbf{w}}'_1, \breve{\mathbf{w}}'_2> \breve{\mathbf{w}}'_1}{< \breve{\mathbf{w}}'_1, \breve{\mathbf{w}}'_1>}  &\cdots 
&|& 
 \breve{\mathbf{W}}'\bar{\mathbf{u}}_1-\frac{< \breve{\mathbf{w}}'_1, \breve{\mathbf{W}}'\bar{\mathbf{u}}_1> \breve{\mathbf{w}}'_1}{< \breve{\mathbf{w}}'_1, \breve{\mathbf{w}}'_1>}&\cdots & \breve{\mathbf{W}}'\bar{\mathbf{u}}_t-\frac{< \breve{\mathbf{w}}'_1, \breve{\mathbf{W}}'\bar{\mathbf{u}}_t> \breve{\mathbf{w}}'_1}{< \breve{\mathbf{w}}'_1, \breve{\mathbf{w}}'_1>}.
\end{array}
            \right)
\end{eqnarray*}
Let $\alpha_i=\frac{< \breve{\mathbf{w}}'_1, \breve{\mathbf{w}}'_i>}{< \breve{\mathbf{w}}'_1, \breve{\mathbf{w}}'_1>}$, $\beta_j=\frac{< \breve{\mathbf{w}}'_1, \breve{\mathbf{W}}'\bar{\mathbf{u}}_j>}{< \breve{\mathbf{w}}'_1, \breve{\mathbf{w}}'_1>}$, $ \breve{\mathbf{w}}'^{(2)}_i= \breve{\mathbf{w}}'_i-\alpha_i \breve{\mathbf{w}}'_1 $, $i\in \{1,2,\cdots r\}, j\in \{1,2,\cdots t\}$, and $ \breve{\mathbf{W}}'^{(2)}=[\breve{\mathbf{w}}'^{(2)}_1,\cdots \breve{\mathbf{w}}'^{(2)}_r]$.  Then
\begin{eqnarray*}
 \breve{\mathbf{W}}'\bar{\mathbf{u}}_j-\beta_j \breve{\mathbf{w}}'_1&=&\sum^r_{i=1} \breve{\mathbf{w}}'_i\bar{u}_{ij}-\beta_j \breve{\mathbf{w}}'_1\\
&=&\sum^r_{i=1}( \breve{\mathbf{w}}'^{(2)}_i+\alpha_i \breve{\mathbf{w}}'_1)
\bar{u}_{ij}-\beta_j \breve{\mathbf{w}}'_1\\
&=& \sum^r_{i=1}  \breve{\mathbf{w}}'^{(2)}_i \bar{u}_{ij}+(\sum^r_{i=1}\alpha_i\bar{u}_{ij}-\beta_j) \breve{\mathbf{w}}'_1 \\
&=&\sum^r_{i=1}  \breve{\mathbf{w}}'^{(2)}_i \bar{u}_{ij}= \breve{\mathbf{W}}'^{(2)}\bar{\mathbf{u}}_j.
\end{eqnarray*}
Hence, $\mathbf{Z}=\left(\begin{array}{cccccccccc}
\mathbf{0}&|&  \breve{\mathbf{w}}'^{(2)}_2&\cdots& \breve{\mathbf{w}}'^{(2)}_r&|& \breve{\mathbf{W}}'^{(2)}\bar{\mathbf{u}}_1&\cdots& \breve{\mathbf{W}}'^{(2)}\bar{\mathbf{u}}_t
\end{array}
            \right)$. Similar result as \eqref{Them1:eq1} can be derived. The second column will be selected from $\mathcal{K}$.

For simplicity, assume $L$ columns have been selected from $\mathcal{K}$, let $P^{\perp}_{\breve{\mathbf{W}}'^{(L)}}$ be the projection onto the orthogonal complement of the columns of $ \breve{\mathbf{W}}'^{(L)}$. Easy get $P^{\perp}_{ \breve{\mathbf{W}}'^{(L)}} \breve{\mathbf{W}}'^{(L)}=0$. We  have
$$
\Vert P^{\perp}_{ \breve{\mathbf{W}}'^{(L)}} (\breve{\mathbf{W}}'\mathbf{u}_j) \Vert_2=\Vert \sum^r_{i=1} P^{\perp}_{ \breve{\mathbf{W}}'^{(L)}} ( \breve{\mathbf{w}}'_i)\bar{u}_{ij} \Vert_2 \leq 
\sum^r_{i=1}\bar{u}_{ij}\Vert  P^{\perp}_{\breve{\mathbf{W}}'^{(L)}} ( \breve{\mathbf{w}}'_i) \Vert_2 \leq \max_i \Vert  P^{\perp}_{ \breve{\mathbf{W}}'^{(L)}} ( \breve{\mathbf{w}}'_i) \Vert_2.$$
It implies that the $L+1$-th column will be selected from $\mathcal{K}$.

Since $\rank(\mat(\breve{\mathbf{M}}))=r$, we can deduce that $\rank(\mat(\breve{\mathbf{W}}))=r$. After select the $r$-th column from $\mathcal{K}$, the residue matrix will degenerate into zero matrix.  The results hence follow.
\end{proof}

\subsection{Algorithms for Solving $\mathbf{H}$}

After column set $\mathcal{K}$ has been determined from Quaternion-SPA, the factor matrix $\breve{\mathbf{W}}= \breve{\mathbf{M}}(:,\mathcal{K})$. The factor matrix $\mathbf{H}\in \mathbb{R}^{r\times n}_+$ can be computed by solving the following quaternion nonnegative least square problem.

Given $\breve{\mathbf{M}}\in \mathbb{H}^{m\times n}_S$, $\breve{\mathbf{W}}\in \mathbb{H}^{m\times r}_S$, solve
\begin{equation}\label{qua_nls}
\min_{\mathbf{H}\in \mathbb{R}^{r\times n}_+}\|\breve{\mathbf{M}}-\breve{\mathbf{W}}\mathbf{H}\|^2_F\doteq \min_{\mathbf{h}_j\in \mathbb{R}^{r}_+}\sum^n_{j=1}\|\breve{\mathbf{m}}_j-\breve{\mathbf{W}}\mathbf{h}_j\|^2_F.
\end{equation}
Let $F(\mathbf{h}_j)=\|\breve{\mathbf{m}}_j-\breve{\mathbf{W}}\mathbf{h}_j\|^2_2$, $j= \{1,2,\cdots n\}$, then $F(\mathbf{h}_j)=\sum\limits^3_{l=0}\|\mathcal{S}_l(\breve{\mathbf{m}}_j)- \mathcal{S}_l(\breve{\mathbf{W}})\mathbf{h}_{j}\|^2_F$. $ \mathbf{h}_j$ in equation \eqref{qua_nls} can be solved by the following optimization model.
\begin{equation}\label{hj}
 \mathbf{h}_j=\mathop{\arg\min}_{\mathbf{h}_j\in \mathbb{R}^{r}_+}F(\mathbf{h}_j)=\sum^3_{l=0}\|\mathcal{S}_l(\breve{\mathbf{m}}_j)- \mathcal{S}_l(\breve{\mathbf{W}})\mathbf{h}_{j}\|^2_F, \quad j=1,2,\cdots, n.
\end{equation}
%The gradient of $F(\mathbf{h}_j)$ is 
%\begin{equation*}
%\frac{\partial F(\mathbf{h}_j)}{\partial \mathbf{h}_j}=2\sum^3_{t=0} [\mathcal{S}^T_l(\breve{\mathbf{W}})\mathcal{S}_t(\breve{\mathbf{M}}_j)- \mathcal{S}^T_l(\breve{\mathbf{W}})\mathcal{S}_t(\breve{\mathbf{W}})\mathbf{h}_j], \quad j\in [n].
%\end{equation*}
%Let $\frac{\partial F(\mathbf{h}_j)}{\partial \mathbf{h}_j}=0$, note that $\mathbf{h}_j\in \mathbb{R}^{r}_+$, we get 
%$$\mathbf{h}_j=\Pi_{\mathbb{R}^r_+}(\big(\sum\limits^3_{l=0}\mathcal{S}^T_l(\breve{\mathbf{W}})\mathcal{S}_l(\breve{\mathbf{W}})\big)^{-1}\sum\limits^3_{l=0}\mathcal{S}^T_l(\breve{\mathbf{W}})\mathcal{S}_l(\breve{\mathbf{M}}_j)),$$
%where $\Pi_{\mathbb{R}^r_+}$ is the projection onto ${\mathbb{R}^r_+}$. 

The following quaternion nonnegative least squares method is proposed in 
\cite{flamant2020quaternion} for model \eqref{hj}.
\begin{algorithm}[H]
\caption{Quaternion Nonnegative Least Squares (QNLS)\cite{flamant2020quaternion} \label{algo:qnls}}
\begin{algorithmic}[1]
\REQUIRE Matrix $\breve{\mathbf{M}} \in \mathbb{H}^{m\times n}_S$, and $\breve{\mathbf{W}} \in \mathbb{H}^{m\times r}_S$. 

\ENSURE  Matrix $\mathbf{H}\in \mathbb{R}^{r\times n}_+$   such that $\min\limits_{\mathbf{H}\in \mathbb{R}^{r\times n}_+}\|\breve{\mathbf{M}}-\breve{\mathbf{W}}\mathbf{H}\|^2_{F}$.

\STATE $\mathbf{A}=\big(\sum\limits^3_{l=0}\mathcal{S}^T_l(\breve{\mathbf{W}})\mathcal{S}_l(\breve{\mathbf{W}})\big)^{-1}$
\FOR{$j$ = 1 : n}
\STATE $\mathbf{h}_j= \mathbf{A}\cdot\sum\limits^3_{l=0}\mathcal{S}^T_l(\breve{\mathbf{W}})\mathcal{S}_l(\breve{\mathbf{m}}_j)$
\STATE \textbf{Projection onto $\mathbb{R}^r_+$:} $\mathbf{h}_j=\Pi_{\mathbb{R}^r_+}(\mathbf{h}_j)$
\ENDFOR
\end{algorithmic}
\end{algorithm}

\begin{remark}
We remark that quaternion  nonnegative least squares needs $(8r^2m+r^3)+(8rmn+2r^2n)+nr$ flops in total. 
\end{remark}
Note that QNLS method is simple but rough. The projection onto $\mathbb{R}^{r\times n}_+$ may lead to many zeros in $\mathbf{H}$. It may cause algorithms fail, for example the quaternion nonnegative matrix factorization (QNMF) proposed in \cite{flamant2020quaternion}. Hence, to address this disadvantage and improve the  solution, an iterative method is given in the following section to compute $\mathbf{H}$. 

\subsubsection{Quaternion Hierarchical Nonnegative Least Squares}

As we know that objective function \eqref{qua_nls} can be further rewritten as, 
\begin{eqnarray*}
F(h_{pj})=\sum\limits^n_{j=1}\sum\limits^3_{l=0}\|\mathcal{S}_l(\breve{\mathbf{m}}_j)- \sum^r_{i\neq p}\mathcal{S}_l(\breve{\mathbf{w}}_i)h_{ij}-\mathcal{S}_l(\breve{\mathbf{w}}_p)h_{pj}\|^2_2.\\
\end{eqnarray*}

%\begin{eqnarray*}
%F(\mathbf{h}_j)&=&\sum\limits^3_{l=0}\|\mathcal{S}_l(\breve{\mathbf{m}}_j)- \mathcal{S}_l(\breve{\mathbf{W}})\mathbf{h}_{j}\|^2_2\\
%&=& \sum\limits^3_{l=0}\|\mathcal{S}_l(\breve{\mathbf{m}}_j)- \sum^r_{i=1}\mathcal{S}_l(\breve{\mathbf{w}}_i)h_{ij}\|^2_2\\
%&=& \sum\limits^3_{l=0}\|\mathcal{S}_l(\breve{\mathbf{m}}_j)- \sum^r_{i\neq p}\mathcal{S}_l(\breve{\mathbf{w}}_i)h_{ij}-\mathcal{S}_l(\breve{\mathbf{w}}_p)h_{pj}\|^2_2.\\
%\end{eqnarray*}
The gradient of  $F(h_{pj})$ is 
\begin{equation*}
\frac{\partial F(h_{pj})}{\partial h_{pj}}=2\sum^3_{l=0}  \mathcal{S}^T_l(\breve{\mathbf{w}}_p)\big(\mathcal{S}_l(\breve{\mathbf{m}}_j)- \sum^r_{i\neq p}\mathcal{S}_l(\breve{\mathbf{w}}_i)h_{ij}-\mathcal{S}_l(\breve{\mathbf{w}}_p)h_{pj}\big).
\end{equation*}
Because of the nonnegative constrains on $\mathbf{H}$, $h_{pj}$ is obtained by 
\begin{equation}\label{hpj}
h_{pj}=\Pi_{\mathbb{R}_+}\Big(\frac{\sum\limits^3_{l=0}\mathcal{S}^T_l(\breve{\mathbf{w}}_p)\big(\mathcal{S}_l(\breve{\mathbf{m}}_j)- \sum\limits^r_{i\neq p}\mathcal{S}_l(\breve{\mathbf{w}}_i)h_{ij}\big)}{\sum\limits^3_{l=0}\mathcal{S}^T_l(\breve{\mathbf{w}}_p)\mathcal{S}_l(\breve{\mathbf{w}}_p)}\Big),\quad j=1,2,\cdots, n.
\end{equation}
Note that the row of $\mathbf{H}$ may equal to zero thanks to the projection $\Pi_{\mathbb{R}_+}$, which will lead to failure of algorithms. To address this problem, one strategy is to use a very small positive number $\xi$ to replace lower bound $0$ in the projection $\Pi_{\mathbb{R}_+}$ in practice. Precisely, we define projection $\Pi_{\mathbb{R}_{+\xi}}$ as follows,
$$
\Pi_{\mathbb{R}_{+\xi}}(x)=\left\{\begin{array}{cl}
x, & \mbox{if} \ x\geq 0,\\
\xi, & \mbox{if} \  x<0,\\
\end{array}\right.
$$
where $\xi\lll 1$ is a very small positive number.
The row of $\mathbf{H}$ can be updated successively by 
\begin{equation}\label{Hrow_0}
\mathbf{H}^{(k+1)}_{p,:}=\Pi_{\mathbb{R}^n_{+\xi}}\Big(\mathbf{H}'^{(k+1)}_{p,:}\Big), 
\end{equation}
where
$$
 \mathbf{H}'^{(k+1)}_{p,:}=\frac{\sum\limits^3_{l=0}\mathcal{S}^T_l(\breve{\mathbf{w}}_p) \mathcal{S}_l(\breve{\mathbf{M}})- \sum\limits^3_{l=0}\sum\limits^{p-1}_{i=1}\mathcal{S}^T_l(\breve{\mathbf{w}}_p)\mathcal{S}_l(\breve{\mathbf{w}}_i)\mathbf{H}^{(k+1)}_{i,:} - \sum\limits^3_{l=0}\sum\limits^r_{i= p+1}\mathcal{S}^T_l(\breve{\mathbf{w}}_p)\mathcal{S}_l(\breve{\mathbf{w}}_i)\mathbf{H}^{(k)}_{i,:}}{\sum\limits^3_{l=0}\mathcal{S}^T_l(\breve{\mathbf{w}}_p)\mathcal{S}_l(\breve{\mathbf{w}}_p)}.
$$
Here $\mathbf{H}_{p,:}$ refers to the $p$-th row of $\mathbf{H}$. 
%\begin{equation}\label{Hrow}
%\mathbf{H}^{(k+1)}_{p,:}=\max\Big(\xi, \mathbf{H}'^{(k+1)}_{p,:}\Big). 
%\end{equation}
The iteration is presented in Algorithm \ref{algo:qhnls}.
\begin{algorithm} 
\caption{Quaternion Hierarchical Nonnegative Least Squares (QHNLS) \label{algo:qhnls}}
\begin{algorithmic}[1]
\REQUIRE Matrix $\breve{\mathbf{M}} \in \mathbb{H}^{m\times n}_S$, and $\breve{\mathbf{W}} \in \mathbb{H}^{m\times r}_S$. Initial matrix $\mathbf{H}^{(0)}\in \mathbb{R}^{r\times n}_+$, maximum iteration $iter$ and stopping criterion $\epsilon$. 

\ENSURE  Matrix $\mathbf{H}$   such that $\min\limits_{\mathbf{H}\geq \xi}\|\breve{\mathbf{M}}-\breve{\mathbf{W}}\mathbf{H}\|^2_{F}$.

\STATE $\mathbf{A}=\sum\limits^3_{l=0}\mathcal{S}^T_l(\breve{\mathbf{W}})\mathcal{S}_l(\breve{\mathbf{W}}) $, $\mathbf{B}=\sum\limits^3_{l=0}\mathcal{S}^T_l(\breve{\mathbf{W}})\mathcal{S}_l(\breve{\mathbf{M}}) $.

\WHILE{$k< iter$ or $\delta<\epsilon_0$}
\FOR{$p$ = 1 : r}
\STATE  $\mathbf{c}=\sum\limits^{p-1}_{i=1}a_{pj}\mathbf{H}^{(k+1)}_{i,:}+\sum\limits^{r}_{i=p+1}a_{pj}\mathbf{H}^{(k)}_{i,:}$
\STATE $\mathbf{H}^{(k+1)}_{p,:}= \frac{\mathbf{B}_{p,:}-\mathbf{c}}{a_{pp}}$
\STATE \textbf{Projection onto $\mathbb{R}^n_+$:} $\mathbf{H}^{(k+1)}_{p,:}=\max(\xi, \mathbf{H}^{(k+1)}_{p,:})$
\ENDFOR
\STATE $\delta=\frac{\|\mathbf{H}^{(k+1)}-\mathbf{H}^{(k)}\|_F}{\|\mathbf{H}^{(1)}-\mathbf{H}^{(0)}\|_F}$
\ENDWHILE
\end{algorithmic}
\end{algorithm}

\begin{remark}
QHNLS needs $(8r^2m+rmn)+(2nr^2+nr)k$ flops in total, here $k$ refers to as the number of iterations.  The initial input for Algorithm \ref{algo:qhnls} can be generated by  Algorithm \ref{algo:qnls} or simply generated randomly. 
\end{remark}  
We note that QHNLS is a block-coordinate descent method. $\mathbf{H}_{p,:}$ is in closed convex set  $\mathbb{R}^{n}_{\xi}$, and the subproblems
\begin{equation}\label{Hop}
\mathbf{H}_{p,:}=\mathop{\arg\min}_{\mathbf{H}_{p,:}\geq \xi}\|\breve{\mathbf{M}}-\sum_{i\neq p}\breve{\mathbf{W}}_{:,i}\mathbf{H}_{i,:}-\breve{\mathbf{W}}_{:,p}\mathbf{H}_{p,:}\|^2_F , \quad p=1,\cdots, m,
\end{equation}
are strictly convex. Here $\mathbf{H}_{p,:}$ and $\mathbf{H}_{:,p}$ refer to the $p$-th row and $p$-th column of $\mathbf{H}$ respectively. 
We can derive that the limit points of the iterates \eqref{Hrow_0} are stationary points, based on the results in\cite{powell1973search,bertsekas1997nonlinear}. The convergence results are presented as follows.
\begin{theorem}
The points of algorithm \ref{algo:qhnls} will converge to  stationary points for every $\xi > 0$.
\end{theorem}
We remark that QHNLS is inspired by hierarchical nonnegative least squares (HNLS) \cite{cichocki2007hierarchical,gillis2012accelerated} that is an exact block-coordinate descent method for solving  nonnegative matrix factorization. The convergence results of Algorithm \ref{algo:qhnls} are similar to HNLS (see Theorem 3 in reference \cite{gillis2008nonnegative}). In the proposed SQMF model,  $\breve{\mathbf{W}}$ is quaternion matrix formed by the columns of quaternion matrix $\breve{\mathbf{M}}$, we only apply QHNLS to compute $\mathbf{H}$.

\section{Numerical Experiments}\label{sec:hyper}

In this section, we test the algorithms on realistic polarization image for application of image representation, and  simulated spectro-polarimetric data set for blind unmixing.  All experiments were run on Intel(R) Core(TM) i5-5200 CPU @2.20GHZ with 8GB of RAM using Matlab. We compare our method to the following state-of-the-art methods.
\begin{enumerate}

\item SPA (Successive projection algorithm\cite{araujo2001successive,gillis2014fast,fu2015self}) is a state-of-the-art separable NMF method. It selects the column with the largest $l_2$ norm then projects all columns of $\mathbf{M}\in \mathbb{R}^{m\times n}$ on the orthogonal complement of the extracted column at each step. Note that SPA can only work on real matrix $\mathbf{M}\in \mathbb{R}^{m\times n}$,  we consider the following  variant for an input quaternion matrix $\breve{\mathbf{M}}\in \mathbb{H}^{m\times n}_S$.
\begin{itemize}
\item  SPA$^*$: We determine the column set  by applying SPA  on the real part of quaternion matrix $\breve{\mathbf{M}}\in \mathbb{H}^{m\times n}_S$, that is $\mathcal{S}_0(\breve{\mathbf{M}})$. 
\end{itemize}

\item QNMF (quaternion NMF, \cite{flamant2020quaternion}) is a state-of-the-art quaternion NMF model  generalized from the standard NMF.  Given a quaternion matrix $\breve{\mathbf{M}}\in \mathbb{H}^{m\times n}_S$, QNMF aims to find $\breve{\mathbf{W}}\in \mathbb{H}^{m\times r}_S$ and $\mathbf{H}\in \mathbb{R}^{r\times n}_+$ such that $\|\breve{\mathbf{M}}-\breve{\mathbf{W}}\mathbf{H}\|$ is minimize. Quaternion alternating least squares algorithm (QALS) extended from standard ALS is proposed to solve QNMF model.

\item ImQNMF: We apply quaternion hierarchical ALS  to replace quaternion ALS to compute $\mathbf{H}$ for the quaternion NMF model, and refer it to as the improved quaternion NMF (ImQNMF).
 
\end{enumerate}

The stopping criterion of QNMF and ImQNMF:
We will use  $\delta = 10^{-4}$, and the maximum number of iterations is 1000. 

\subsection{Polarization Image Representation}
The polarization image "cover" is from polarization image dataset containing 40 images with several scenes, constructed by Qiu and etc. in \cite{qiu2021linear}. The "cover" picture has $1024\times 1024$ pixels, and each pixel contains information  of intensity and polarization. That is,
$$
\breve{\mathbf{T}}=\mathcal{S}_0(\breve{\mathbf{T}})+\mathtt{i}\mathcal{S}_1(\breve{\mathbf{T}})+\mathtt{j}\mathcal{S}_2(\breve{\mathbf{T}})+\mathtt{k}\mathcal{S}_3(\breve{\mathbf{T}})\in \mathbb{H}^{1024\times 1024}_S,
$$
where $\mathcal{S}_0(\breve{\mathbf{T}})\in \mathbb{R}^{1024\times 1024}_+$ represents the total intensity, and $(\mathcal{S}_1(\breve{\mathbf{T}}),\mathcal{S}_2(\breve{\mathbf{T}}),\mathcal{S}_3(\breve{\mathbf{T}}))$ stands for its polarization.  
We show the image "cover" captured by a colour camera and polarization image sensor in Fig.\ref{polarimage}. The top row of Fig.\ref{polarimage} only shows the intensity of "cover" captured by the colour camera with polarization filter at $0^{\circ}$. 
The total intensity $\mathcal{S}_0(\breve{\mathbf{T}})$ and polarization $(\mathcal{S}_1(\breve{\mathbf{T}}),\mathcal{S}_2(\breve{\mathbf{T}}))$ measured by four intensities with linear polarizers oriented at $0^{\circ}, 45^{\circ}, 90^{\circ},135^{\circ}$, are shown at the the bottom row of Fig.\ref{polarimage}. Since $\mathcal{S}_3(\breve{\mathbf{T}})=0$ here, we do not show it in Fig.\ref{polarimage}. We have the following observations: 
\begin{itemize}
\item The total intensities of background and the cover are almost the same, except that of the edge of the cover. 

\item The differences between polarization $\big(\mathcal{S}_1(\breve{\mathbf{T}}), \mathcal{S}_2(\breve{\mathbf{T}})\big)$ of background and cover are significant. Especially, the center of cover has higher polarization values than the other parts.
\end{itemize}

\begin{figure} 
\includegraphics[height=5cm,width=\textwidth]{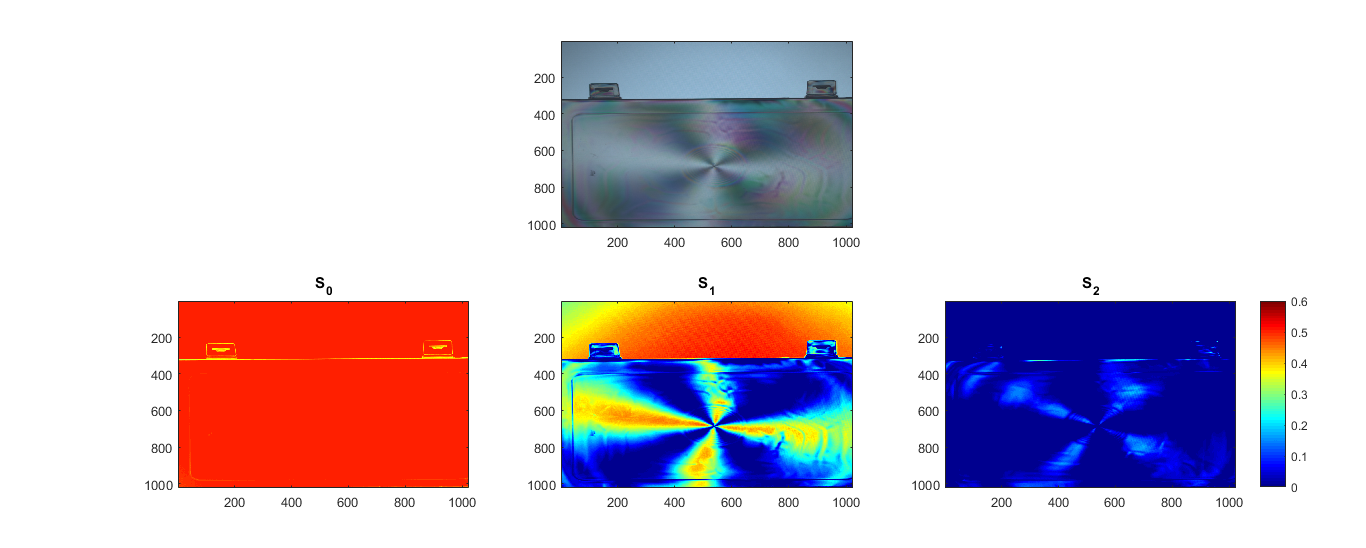}
\caption{Top: image "cover"; Bottom: total intensity $\mathcal{S}_0(\breve{\mathbf{T}})$, polarization $\mathcal{S}_1(\breve{\mathbf{T}})$ and $\mathcal{S}_2(\breve{\mathbf{T}})$ }\label{polarimage}
\end{figure}

From the observation, the edge and center of the picture "cover" show the main characteristic. As we know that QSPA and SPA$^*$ are able to select data points that can be regarded as feature points, it will be interesting to see their identification ability and the reconstruction effect for the "cover" image.  Hence, we first divide image "cover" $\breve{\mathbf{T}}$ into 256 small image blocks, and each block has $64\times 64$ pixels. The data matrix $\breve{\mathbf{M}}\in\mathbb{H}^{4096\times 256}_S $ is then generated by these 256 blocks; that is, each column of $\breve{\mathbf{M}}$ is generated by vectorizing every image block. We test QSPA and SPA$^*$ on $\breve{\mathbf{M}}$, and determine $r$ columns, which correspond to the $r$ key image blocks. 
\begin{figure} 
\includegraphics[height=6cm,width=\textwidth]{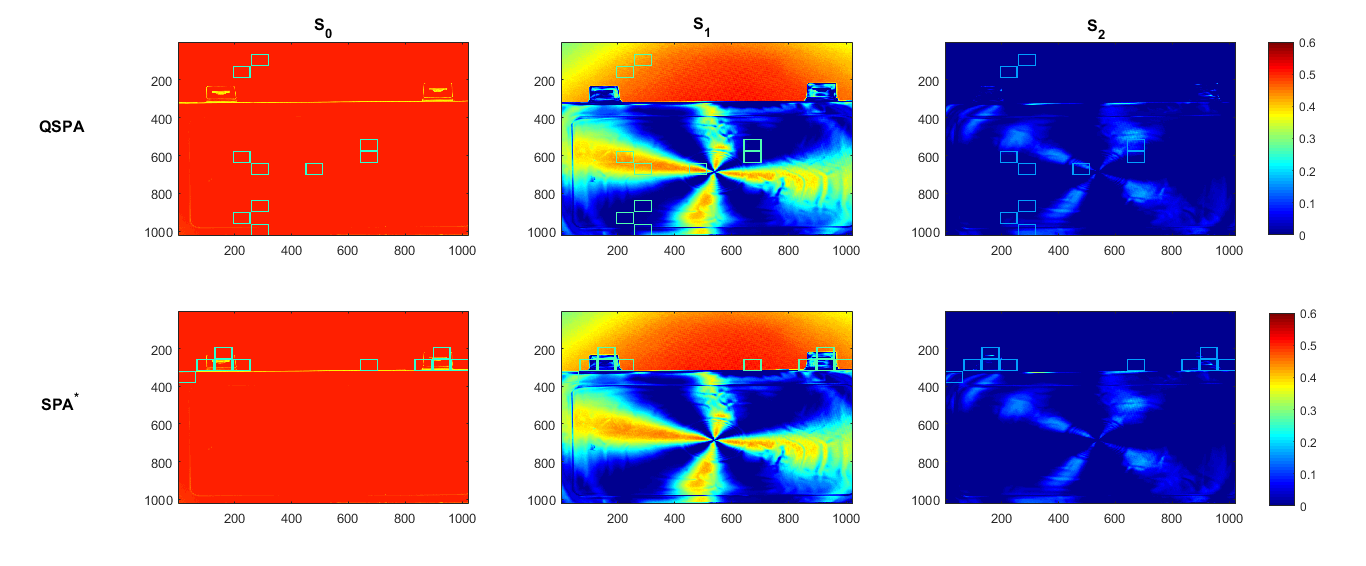}
\caption{10 key blocks identified by QSPA and SPA$^*$.}\label{block-selected}
\end{figure}

In Fig. \ref{block-selected}, we show $r=10$ images blocks determined by QSPA and SPA$^*$. 
It is interesting to observe that the areas determined by QSPA are in the background and the center of the cover, while SPA$^*$ selects the edge of the cover. The results are surprising coincident with the observation of the characteristics of total intensity $\mathcal{S}_0$ and polarization 
$(\mathcal{S}_1,\mathcal{S}_2)$. We remark that QNMF and ImQNMF are not able to identify critical regions  that contain latent features from a polarization image.  

In Table \ref{table:Po_result}, we report the approximations of the reconstructed image of all the methods, which are defined as follows.
\begin{enumerate}
\item The total relative approximation in percentage:
\begin{equation}\label{appro}
Appro=100-100*\dfrac{\|\breve{\mathbf{M}}-\breve{\mathbf{W}}\mathbf{H}\|_F}{\|\breve{\mathbf{M}}\|_F}.
\end{equation}
  
\item 
The relative approximation on intensity ($S_0$) and polarization information ($S_1,S_2,S_3$) in percentage are defined as
 \begin{equation} \label{appro1}
\text{app-}s_l= 100- 100*\frac{ \|\mathcal{S}_l(\breve{\mathbf{M}})-\mathcal{S}_l(\breve{\mathbf{W}})\mathbf{H}\|_F}{\|\mathcal{S}_l(\breve{\mathbf{M}})\|_F}, \quad l=0,1,2,3.
\end{equation}
\end{enumerate}
We remark that for QNMF, ten initials are generated randomly, and the results with the highest "Appro" value are reported. From Table \ref{table:Po_result}, we observe that: 
\begin{itemize}
\item In terms of the relative approximation of all components, as $r$, the number of features increases, so make the approximations of all the methods. It is not surprising that ImQNMF gets the highest values, since ImQNMF does not consider extra constraints. QSPA receives the second best results, while SPA$^*$ has the worst ones. We note that QNMF fails when $r=30,50$ for all initial guesses.

\item In terms of running time, SPA$^*$ is the fastest method, which needs less than 1 second. The proposed QSPA is the second fast, and ImQNMF is the slowest method. Especially when $r=50$, ImQNMF is almost $10^5$ times slower than QSPA, and $10^6$ times slower than SPA$^*$.
\end{itemize}
It is worth mentioning that,  QSPA  and SPA  only need to store the indexes of important columns and the weight matrix $\mathbf{H}$, to represent the polarization image. Precisely, for a given $\breve{\mathbf{M}}\in \mathbb{H}^{m\times n}_S$ and number of features $r$, to represent the factor matrices $\breve{\mathbf{W}}\in \mathbb{H}^{m\times r}_S$ and $\mathbf{H}\in \mathbb{R}^{r\times n}$, QSPA  and SPA only store $r+rn$ parameters,  while QNMF and ImQNMF need $4rm+rn$.

Visualization of reconstructed images by $r=10$ features is shown in Fig.  \ref{polarresult}. It is interesting to find the results of QSPA capture the main characteristics of the polarization image "cover". The results of QNMF and ImQNMF are more smooth, compared to that of SPA$^*$ and QSPA. We also notice that SPA$^*$ is not able to recover  $\mathcal{S}_2(\breve{\mathbf{T}})$. 

In Appendix, we show more visualization results on $r=30,50$.

\begin{figure} 
\includegraphics[height=8.5cm,width=\textwidth]{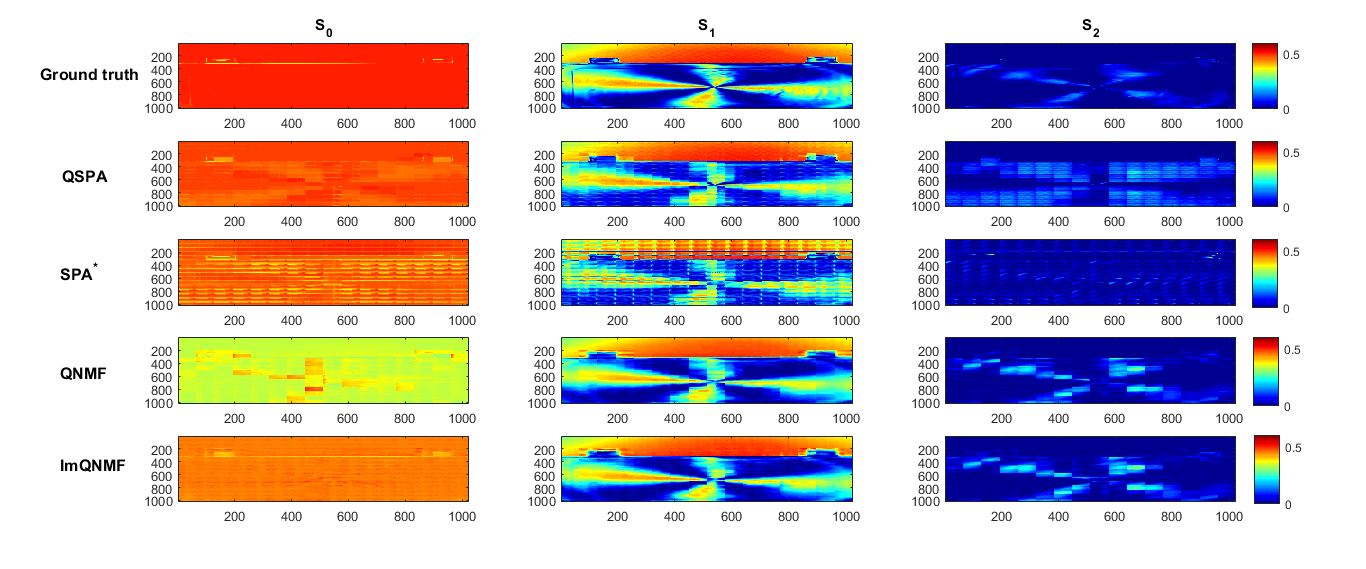}
\caption{Reconstruction results $(\mathcal{S}_0(\breve{\mathbf{T}}),\mathcal{S}_1(\breve{\mathbf{T}}),\mathcal{S}_2(\breve{\mathbf{T}}))$ of all methods $(r=10)$.}\label{polarresult}
\end{figure}

\begin{table} 
  \centering 
  \begin{tabular}{|c||c||c|c|c|c||c|}
  \hline
 ~& r&Appro & app-$s_0$ & app-$s_1$ & app-$s_2$ & Time(s)\\  
 \hline
  \multirow{3}*{QSPA} &10&88.65& 96.89&84.69& 27.97&0.55\\
  \cline{2-7}&30&92.95&97.76&89.37&32.46&1.79\\
   \cline{2-7}&50&95.64&99.08&93.52&54.78&3.26\\
 \hline
 \hline
 \multirow{3}*{SPA$^*$}&10&84.03&93.84&72.55&4.43&0.01\\
   \cline{2-7}&30&90.89& 97.25&85.18&24.84&0.04\\
   \cline{2-7}&50&91.87& 97.77&86.35&37.38&0.10\\
 \hline
 \hline
 \multirow{3}*{QNMF}&10&91.28&92.98&88.75&56.41&2036\\
  \cline{2-7}&30&-&-&-&-&-\\
   \cline{2-7}&50&-&-&-&-&-\\
 \hline
 \hline
 \multirow{3}*{ImQNMF}&10&94.20&97.90& 90.46&56.37&2048\\
   \cline{2-7}&30&96.95&98.79&95.29&73.79&87155\\
    \cline{2-7}&50&97.91& 99.00&96.84&83.12&296840\\
 \hline
 \end{tabular}
 \caption{Numerical results (in percent) for 
 cover spectral-polarimetric image.}\label{table:Po_result}
\end{table}

We conclude the advantages of QSPA in the application of image representation as follows:  QSPA can capture important characteristics from a polarization image. It can reconstruct original image well when the number of features is big enough. At the mean time, it runs very fast. 

\subsection{Spectro-Polarimetric Image Unmixing}

We simulate spectro-polarimetric data sets from the Urban HSI \cite{guo2009l1}. 
Assume $(\mathbf{W}^0,\mathbf{H}^0)$ be the unmixing ground truth of HSI,   $\mathbf{W}^0\in \mathbb{R}^{m\times r_0}_+$ stands for the endmember matrix and $\mathbf{H}^0 \in \mathbb{R}^{r_0\times n}_+$ represents the abundance matrix of all endmembers. 
The spectro-polarimetric data sets are simulated by
\begin{equation}\label{simulateRule}
\breve{\mathbf{M}}=\breve{\mathbf{M}}^*+\mathbf{N}
\end{equation}
where $\breve{\mathbf{M}}^*=\breve{\mathbf{W}}^*\mathbf{H}^*\in \mathbb{H}^{m\times n}_S$, and $\mathbf{N}\in \mathbb{H}^{m\times n}$ is noise matrix. 
\begin{itemize}
\item The entries of polarimetric matrix $\breve{\mathbf{W}}^*\in \mathbb{H}^{m\times r}_S$ is generated by
$$
  \breve{\mathbf{W}}^*=\mathcal{S}_0(\breve{\mathbf{W}}^*)+\mathtt{i}\mathcal{S}_1(\breve{\mathbf{W}}^*)+\mathtt{j}\mathcal{S}_2(\breve{\mathbf{W}}^*)+\mathtt{k}\mathcal{S}_3(\breve{\mathbf{W}}^*),
  $$
where $\mathcal{S}_0(\breve{\mathbf{W}}^*)\in \mathbb{R}^{m\times r}$ is simulated from endmember matrix $\mathbf{W}^0$, and 
\begin{eqnarray*}
\mathcal{S}_1(\breve{\mathbf{W}}^*)&=&\phi\mathcal{S}_0(\breve{\mathbf{W}}^*) \cos{\mathbf{D}}_{\bm{\alpha}}\cos{\mathbf{D}}_{\bm{\beta}}, \\
 \mathcal{S}_2(\breve{\mathbf{W}}^*)&=&\phi  \mathcal{S}_0(\breve{\mathbf{W}}^*) \sin{\mathbf{D}}_{\bm{\alpha}}\cos{\mathbf{D}}_{\bm{\beta}}, \\
 \mathcal{S}_3(\breve{\mathbf{W}}^*)&=&\phi  \mathcal{S}_0(\breve{\mathbf{W}}^*)\sin{\mathbf{D}}_{\bm{\beta}}.
\end{eqnarray*}
$\phi$ is the degree of polarization. Here we set $\phi=1$, i.e., fully polarized.  
$\mathbf{D}_{\bm{\alpha}}=Diag(\bm{\alpha})$, $\mathbf{D}_{\bm{\beta}}=Diag(\bm{\beta})$ are diagonal matrices with the elements of vectors $\bm{\alpha}$ and $\bm{\beta}$ on the main diagonal, respectively.  Here
$\bm{\alpha}\in \mathbb{R}^{r\times 1}$  and  $\bm{\beta} \in \mathbb{R}^{r\times 1}$ are generated uniformly at random in the interval [0,1]: $\bm{\alpha}=(2*rand(1,r)-1)*\pi$ and $\bm{\beta}=(2*rand(1,r)-1)*\pi$. 

\item $\mathbf{H}^*\in \mathbb{R}^{r\times n}_+$ is generated from the abundance matrix $\mathbf{H}^0$ of HSI data sets.

\item The noise matrix $\mathbf{N}\in \mathbb{H}^{m\times n}$ is generated at random normal distribution, that is $\{\mathcal{S}_i(\mathbf{N})\}^3_{i=0}$ are generated by the \textit{randn} function of MATLAB. $\mathbf{N}$ is normalized such that $\|\mathbf{N}\|_F=\epsilon\|\breve{\mathbf{M}}^*\|_F$. $\epsilon$ is the noise level defined as $\epsilon=\frac{\|\mathbf{N}\|_F}{\|\breve{\mathbf{M}}^*\|_F}$.  
\end{itemize}

To evaluate the quality of solution $(\breve{\mathbf{W}},\mathbf{H})$ computed by a method, the total relative approximation "Appro" of \eqref{appro} and relative approximation on the components "app-$s_l$'' $(l=0,1,2,3)$ in  \eqref{appro1}  will be reported. We also present the relative approximation to ground truth $(\breve{\mathbf{W}}^*,\mathbf{H}^*)$ in percentage defined as follows.
$$
appW=100-100*\frac{\min_{\pi_w}\|\breve{\mathbf{W}}^*-\breve{\mathbf{W}}(:,\pi_w)\|_F}{\|\breve{\mathbf{W}}^*\|_F},\quad appH=100-100*\frac{\min_{\pi_h}\|\mathbf{H}^*-\mathbf{H}(:,\pi_h)\|_F}{\|\mathbf{H}^*\|_F},
$$
where $\pi_w$, and $\pi_h$ are permutations.

We test all the methods on three noise levels $\epsilon=[0,0.05,0.1]$ on the simulated data sets.  For each noise level, ten such matrices are generated. For each matrix, ten initials are generated randomly for QNMF method, and we report the results of that has the best "Appro" value.

\subsubsection{The Simulated Spectro-Polarimetric Urban Dataset}

The Urban HSI \cite{guo2009l1} is taken from Hyper-spectral Digital Imagery Collection Experiment (HYDICE) air-borne sensors and contains 162 clean spectral bands where each image has dimension $307\times 307$. Therefore the "spectral $\times$ pixels" data matrix $\mathbf{M}^0$  has dimensions 162 by 94249. The Urban data is mainly composed of 6 types of materials:
Road; Grass;  Tree;  Roof;  Metal; Dirt  (for more details, see \cite{guo2009l1}). 
Its ground truth $(\mathbf{W}^0,\mathbf{H}^0)$ are  
\begin{eqnarray*}
\mathbf{W}^0&=&[\mathbf{w}^0_1~|~ \mathbf{w}^0_2~|~ \mathbf{w}^0_3~|~\mathbf{w}^0_4~|~ \mathbf{w}^0_5~| \mathbf{w}^0_6]\in \mathbb{R}^{162\times 6}_+,  \\ \mathbf{H}^0&=&[(\mathbf{h}^0_1)^T~|~(\mathbf{h}^0_2)^T~|~(\mathbf{h}^0_3)^T~|~(\mathbf{h}^0_4)^T~|~(\mathbf{h}^0_5)^T~|~(\mathbf{h}^0_6)^T]^T\in \mathbb{R}^{6\times 94249}_+, 
\end{eqnarray*}
where $\mathbf{w}^0_j\in \mathbb{R}^{162\times 1}_+$, $\mathbf{h}^0_j\in \mathbb{R}^{94249\times 1}_+$, $j=1,2,\cdots,6.$

\begin{enumerate}[label=(\roman*)]

\item \textbf{Test 2.1: 6 sources spectro-polarimetric Urban dataset.  }

We simulate six sources spectro-polarimetric dataset by \eqref{simulateRule} based on Urban HSI with 
$$
\mathcal{S}_0(\breve{\mathbf{W}}^*)=\mathbf{W}^0;\quad \mathbf{H}^*=\mathbf{H}^0.
$$ 

\begin{table} 
  \centering 
  \begin{tabular}{|c|c||c|c|c|c|c||c|c||c|}
  \hline
 & Noise level& Appro & app-$s_0$ & app-$s_1$ & app-$s_2$ & app-$s_3$ & appW& appH& Time(s)\\  
  \hline
  \multirow{4}*{QSPA}&0\% &100& 100  & 100& 100& 100& 100&100&148.41\\
   \cline{2-10}&5\%& 93.59 & 95.72 & 89.50&87.39&91.25& 94.82&96.26&150.83\\
   \cline{2-10}&10\%&86.95&91.15 &78.46 &74.72&82.36&85.86&86.36&148.05\\
%    \cline{2-10}&30\%&60.65&72.36 &32.60 &19.53&46.00&50.73&42.34&147.63\\
 \hline
 \hline
   \multirow{4}*{SPA$^*$}&0\% & 100 &100& 100& 100& 100&100&100&0.11\\
   \cline{2-10}&5\%&68.28 &79.10&50.52&36.30&62.73&48.94&38.22&0.11\\
   \cline{2-10}&10\%&66.10 &77.52&46.89&31.43&58.93&51.09&38.77&0.11\\
 %   \cline{2-10}&30\%&41.11 &57.18&3.79&-9.14&25.09&26.96&19.13&0.11\\
 \hline
 \hline
    \multirow{4}*{QNMF}&0\% & 95.02 &94.86&93.91&94.09&95.67&79.08 &74.33 &2194\\
   \cline{2-10}&5\%& 92.52 &93.53&89.32 &87.58& 91.58&78.03&73.67&2157\\
   \cline{2-10}&10\%&  88.50 & 91.17 & 82.14& 78.91&85.39&78.56&74.30&2106\\
%    \cline{2-10}&30\%& 70.70 & 78.99& 50.46& 40.73& 58.81& 77.34& 70.92&2056\\
 \hline
 \hline
  \multirow{4}*{ImQNMF}&0\% &99.38&99.53 &99.17&99.08&99.16&91.55&86.61&1357 \\
   \cline{2-10}&5\%&  94.99& 96.65& 91.80& 90.16& 93.16& 93.95& 88.84&2890\\
   \cline{2-10}&10\%& 89.92&93.27& 83.28&80.14&86.27& 91.36& 87.31&2669\\
%    \cline{2-10}&30\%&  71.36& 80.12& 51.16&41.16& 59.18& 90.51& 80.92&2631\\
 \hline
 \end{tabular}
 \caption{Numerical results (in percent) for 6 sources spectro-polarimetric Urban dataset ($r=6$).}\label{table:Urbansource06}
\end{table}

The average quality measures in percent is reported in Table \ref{table:Urbansource06}. We have the following observations:
\begin{itemize}
\item[$-$] Regarding the relative approximation on all components, both QSPA and SPA$^*$ can achieve $100\%$ in the noiseless case.  As the noise level increases, the relative approximation of all methods decreases. On all noise levels, QSPA gets better results than SPA$^*$. ImQNMF overperforms QNMF; in a certain sense it validates the effectiveness of QHNLS.  It is reasonable that ImQNMF performs the best as the noise level $\epsilon \geq 5\%$, since it aims to find factor matrices without separability constraints.

\item[$-$]  In terms of recovering the ground truth $(\breve{\mathbf{W}}^*,\mathbf{H}^*)$, QSPA performs the best at the noise level $\epsilon\leq 5\%$.  ImQNMF is the best at the noise level $\epsilon= 10\%$. The reason is that the matrix at the noise level $\epsilon= 10\%$ is already far from a separable quaternion matrix. Worth  noting that SPA$^*$ can recover ground truth in noiseless case; the reason is the real component  $\mathcal{S}_0(\breve{\mathbf{M}}^*)$ is 6-separable matrix, i.e., the spectrum already contains 6 endmembers.

\item[$-$] We also report average running time for all the methods. SPA$^{*}$ only needs an average of $0.11$ seconds, and QSPA spends around $148$ seconds. Both QNMF and ImQNMF are very slow, take more than $2000$ seconds in the most cases.  
\end{itemize}  

In Appendix, we show some visual results at noise level $\epsilon= 5\%$ for one arbitrary group of simulated polarimetric parameters $\{\bm{\alpha},\bm{\beta}\}$ . 

\item \textbf{Test 2.2: 10 sources spectro-polarimetric Urban dataset.}

As we know that the "different objects with same intensity" problem is challenging in HSI blind unmixing. In the following simulated data set, we will show that the problem can be well solved with the help of polarization.  We simulate a polarized Urban dataset containing ten objects, i.e., $r=10$, with some different objects having the same intensity.

\begin{itemize}
\item  The intensity of source matrix $\mathcal{S}_0(\breve{\mathbf{W}}^*):$  
$$
\mathcal{S}_0(\breve{\mathbf{W}}^*)=[\mathbf{W}^0|~ \mathbf{w}^0_1~|~ \mathbf{w}^0_1~|~ \mathbf{w}^0_3~|~ \mathbf{w} ^0_4]\in \mathbb{R}^{162\times 10}_+,
$$ 
\end{itemize}
 
\begin{itemize}
\item The activation matrix $\mathbf{H}^*\in \mathbb{R}^{10\times 94249}_+: $ 
$$
[(\mathbf{h}^*_1)^T~|~(\mathbf{h}^*_2)^T~|~(\mathbf{h}^*_3)^T|~(\mathbf{h}^*_4)^T~|~(\mathbf{h}^*_5)^T~|~(\mathbf{h}^*_6)^T~|~(\mathbf{h}^*_7)^T~|~(\mathbf{h}^*_8)^T~|~(\mathbf{h}^*_9)^T~|~(\mathbf{h}^*_{10})^T]
$$
with$$
\mathbf{h}^*_1+\mathbf{h}^*_7+\mathbf{h}^*_8=\mathbf{h}^0_1; \quad
 \mathbf{h}^*_2=\mathbf{h}^0_2; \quad
\mathbf{h}^*_3+\mathbf{h}^*_9=\mathbf{h}^0_3;\quad
\mathbf{h}^*_4+\mathbf{h}^*_{10}=\mathbf{h}^0_4;\quad
\mathbf{h}^*_5=\mathbf{h}^0_5; \quad
\mathbf{h}^*_6=\mathbf{h}^0_6;
$$
where $\mathbf{h}^*_7$ takes the last 500 positions in $\mathbf{h}^0_1$ where the values are 1, and the last 1000 positions in $\mathbf{h}^0_1$ whose values are between 0 to 1. 

$\mathbf{h}^*_8$ takes the first 500 positions in $\mathbf{h}^0_1$ value 1, and the first 1000 positions in $\mathbf{h}^0_1$ whose values are between 0 to 1.
 
$\mathbf{h}^*_9$ takes the last 1000 positions in $\mathbf{h}^0_3$ value 1, and  the last 1000 positions in $\mathbf{h}^0_3$ whose values are between 0 to 1. 

$\mathbf{h}^*_{10}$ takes the last 300 positions in $\mathbf{h}^0_4$ where the values are 1, and  the last 1000 positions in $\mathbf{h}^0_4$ whose values are between 0 to 1. 
\end{itemize}
 \end{enumerate}

\begin{table} 
  \centering 
  \begin{tabular}{|c|c||c|c|c|c|c||c|c||c|}
  \hline
 & Noise level& Appro & app-$s_0$ & app-$s_1$ & app-$s_2$ & app-$s_3$ & appW&appH& Time(s)\\  
  \hline
  \multirow{4}*{QSPA}&0\% & 100& 100 & 100& 100& 100 & 100& 100&217.72\\
   \cline{2-10}&5\%& 93.64 & 95.75 & 87.11 & 87.90 & 92.52 & 90.57 & 77.67 &216.87\\
   \cline{2-10}&10\%&  86.09 & 90.72 &70.19& 74.48 & 83.68 & 75.75 & 50.28 &216.18\\
%    \cline{2-10}&30\%& 61.55& 73.37 &18.91&23.98&52.48&47.20&18.04&216.46\\
 \hline
 \hline
   \multirow{4}*{SPA$^*$}&0\% & 87.45&91.47&79.10&79.80&82.75&--&--&0.27\\
   \cline{2-10}&5\%& 73.56&86.35&44.85&44.52&74.00&45.62&12.85&0.23\\
   \cline{2-10}&10\%&70.78&83.72&42.54&37.92&70.13&46.27&14.19&0.24\\
 %   \cline{2-10}&30\%&54.66&69.62&3.57&11.76&46.05&29.50&4.14&0.26\\
 \hline
 \hline
    \multirow{4}*{QNMF}&0\% & 77.11&74.37&85.05&84.72&79.05&62.75&47.93&5633\\
   \cline{2-10}&5\%& 77.50&75.38&79.41&78.41&79.66&53.51&40.24&5614\\
   \cline{2-10}&10\%& 73.08&71.57&73.35&70.07&72.88&50.83&38.80&5601\\
%    \cline{2-10}&30\%&  65.68 &71.02& 37.80 & 39.67 & 59.13 &  47.56 & 31.32 &5609\\
 \hline
 \hline
  \multirow{4}*{ImQNMF}&0\% & 99.19 & 99.42 & 98.66 & 98.22& 99.14& 68.22& 56.82&13370\\
   \cline{2-10}&5\%&  94.93& 96.60& 89.81& 90.36& 93.99& 67.61&54.42&15013\\
   \cline{2-10}&10\%&  90.07 & 93.35& 79.77 & 80.95 & 88.29&72.75& 60.68&14973\\
%    \cline{2-10}&30\%&  71.39 & 80.11 & 39.45&43.07& 64.99 & 68.32& 49.31 &14847\\
 \hline
 \end{tabular}
 \caption{Numerical results (in percent) for 10 sources spectro-polarimetric Urban dataset ($r=10$).}\label{table:Urban10}
\end{table}

\begin{figure} 
 \includegraphics[height=5cm,width=0.5\textwidth]{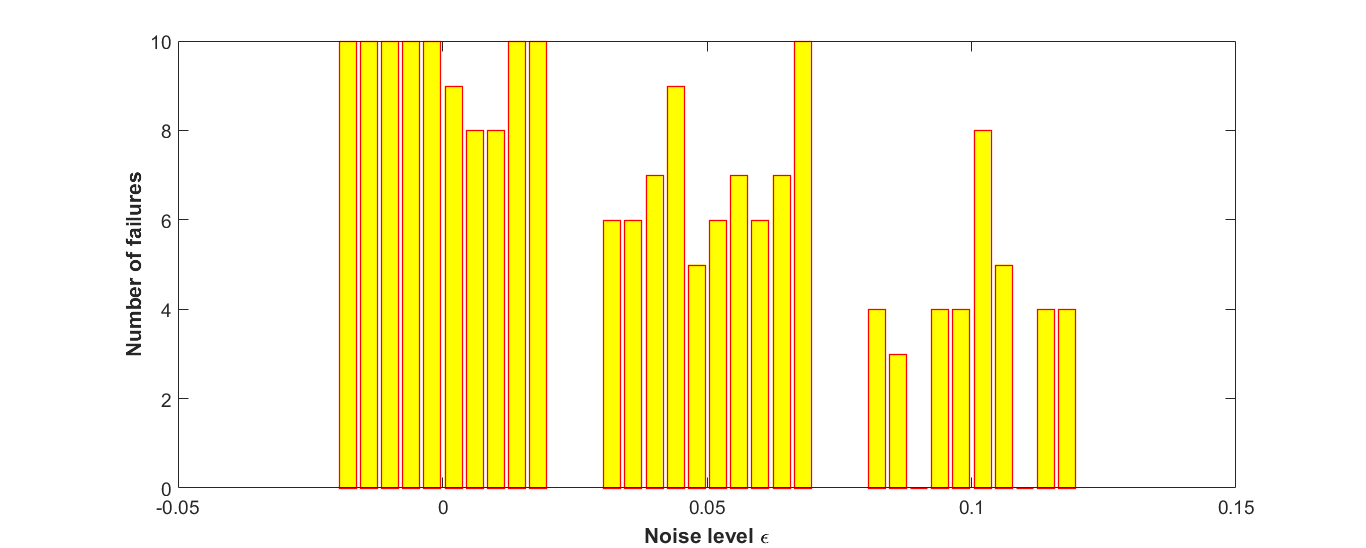}  \includegraphics[height=5cm,width=0.5\textwidth]{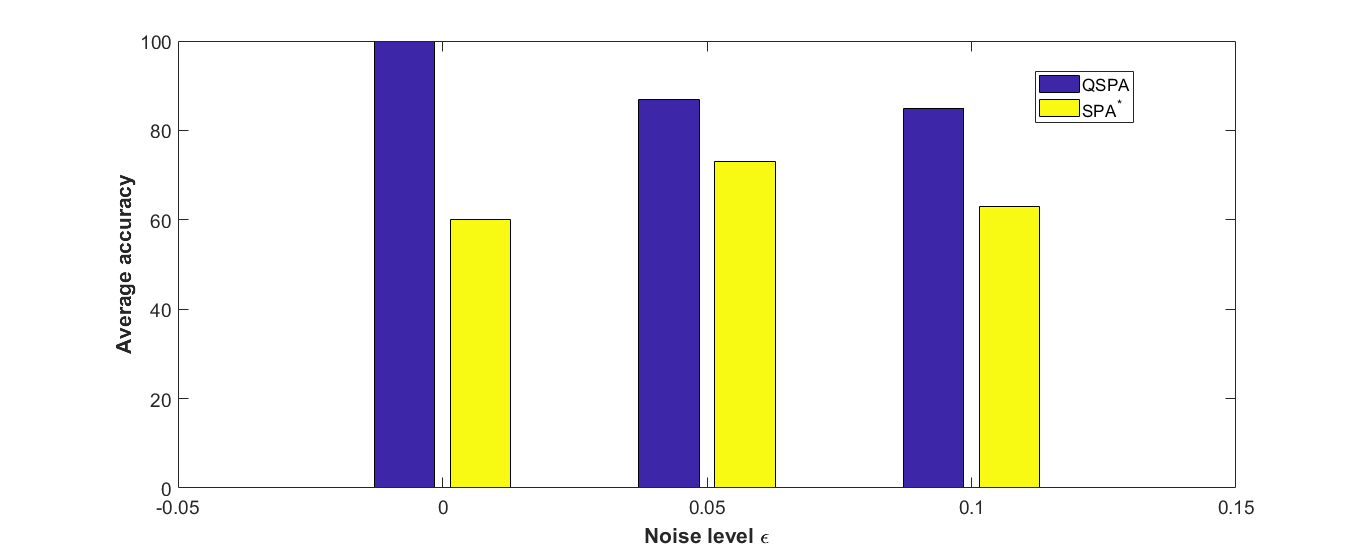}
 \caption{Urban: Left: The number of failures of QNMF; ~~~Right: Accuracy of QSPA and SPA$^*$ }\label{urban_acc}
 \end{figure}

We report the average quality measures in percent in Table \ref{table:Urban10} and observe:
\begin{itemize}
\item[$-$] In terms of the relative approximation on all components, only  QSPA  can achieve $100\%$ in a noiseless case.  The relative approximation of all methods decreases as the noise level increases. When the noise level $\epsilon \geq 5\%$,  ImQNMF has the best results. It is also nice to see that QSPA obtains the second best at the noise level $\epsilon = 5\%, 10\%$. We notice that SPA$^*$ does not perform well here. 
% At noise level $\epsilon = 30\%$,  QSPA does not perform as well as it at a lower noise level since the matrix $\breve{\mathbf{M}}$ is far from separable quaternion matrix. 

\item[$-$]  Regarding recovering the ground truth $(\breve{\mathbf{W}}^*,\mathbf{H}^*)$,  QSPA performs the best. We remark that in the noiseless case, QSPA 
can recover the factor matrices, while SPA$^*$ could not work because the rank of the real component of $\breve{\mathbf{M}}$ is 6, less than $r=10$, the number of the simulated sources.

\item[$-$] The average running time for all the methods are also presented. SPA$^{*}$ only needs average $0.2$ seconds, and QSPA takes around $216$ seconds. However QNMF takes more than 5600 seconds and ImQNMF needs more than 13000 seconds, much slower than the other two methods.
\end{itemize}

Since QNMF fails for some initial guesses in this ten sources spectro-polarimetric dataset, we report the number of its failures among ten initials for each generated matrix $\breve{\mathbf{M}}$ at all the noise levels in Fig. \ref{urban_acc}. We find that QNMF fails more frequently at a negligible noise level, thanks to the singularity of $\breve{\mathbf{W}}$ caused by too many zeros of $\mathbf{H}$ by QNLS. Especially in the noiseless case, it fails at 97 initials in 100 in total. On the right of Fig. \ref{urban_acc}, we report the identification accuracy of SPA$^*$ and QSPA  defined as follows.
\begin{equation} \label{accuracy}
  \text{accuracy} =
  \frac{|\mathcal{K}^* \cap\mathcal{K}| }{|\mathcal{K}^*|},
\end{equation}
where $\mathcal{K}^*$ is the true column indices of pure sources in $\breve{\mathbf{M}}^*$. It shows the proportion of column indices that are correctly identified. It is not surprising to see the accuracies of QSPA are better than SPA$^*$ at all levels of noise. 

The visual results under one arbitrary group of simulated polarimetric parameters $\{\bm{\alpha},\bm{\beta}\}$ are presented in Figs. \ref{urbandata}-\ref{S3_urban_r10}.
 The simulated spectro-polarimetric image data $\breve{\mathbf{M}}=\breve{\mathbf{W}}\mathbf{H}$ for three distinct wavelengths indices $t=30,90,150$ are shown in 
Fig.\ref{urbandata}.   The sources and the corresponding activations factors from all the methods at noise level $5\%$ are shown in Figs.\ref{H_urban_r10}-\ref{S3_urban_r10}. Visually from Fig. \ref{H_urban_r10}, it is exciting that QSPA can almost recover the activation matrix $\mathbf{H}$ at the  noise level $5\%$. From Figs. \ref{S0_urban_r10}-\ref{S3_urban_r10}, since QSPA and SPA$^*$ determine $\breve{\mathbf{W}}$ from the noisy data matrix $\breve{\mathbf{M}}$, the components of the source matrix $\breve{\mathbf{W}}$ from QSPA and SPA$^*$ have some noise compared to the ground truth $\breve{\mathbf{W}}^*$. Compared to the other three methods, visually, the components of the source matrix $\breve{\mathbf{W}}$ from QSPA are much closer to ground truth $\breve{\mathbf{W}}^*$. The visual results indicate that QSPA overperforms the others to get factor matrices $(\breve{\mathbf{W}},\mathbf{H})$, which verify the results shown in Table \ref{table:Urban10}. 
 
\begin{figure} 
\includegraphics[height=8cm,width=\textwidth]{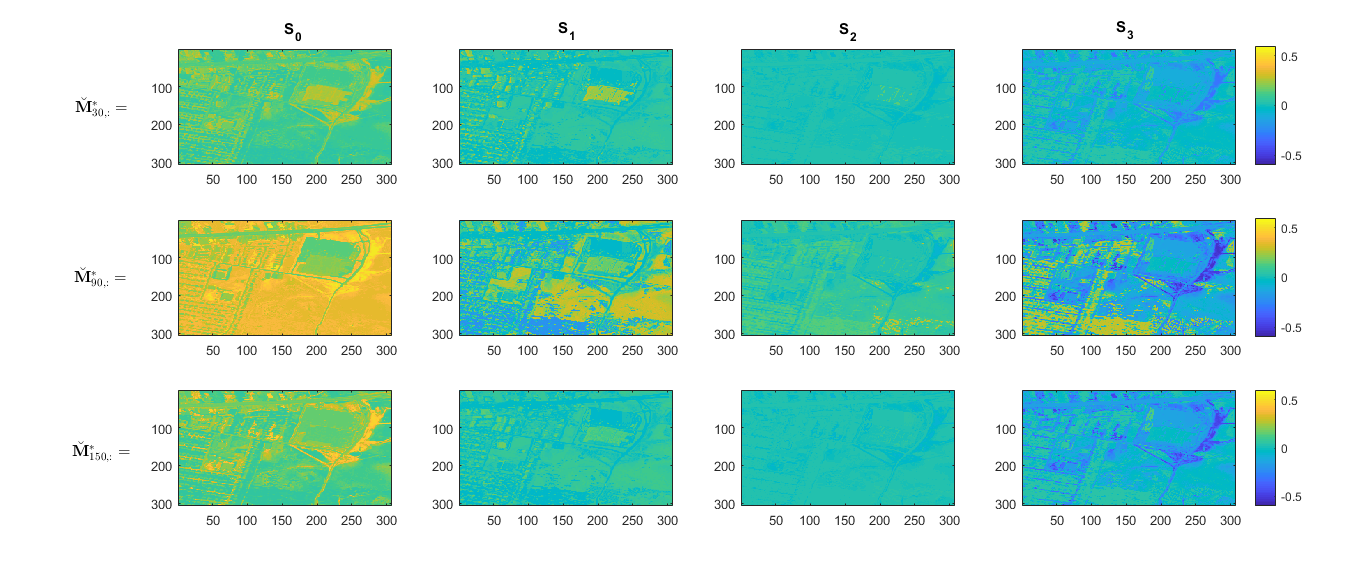}
\caption{2D intensity $\mathcal{S}_0(\breve{\mathbf{M}}^*)$ and polarization $\big(\mathcal{S}_1(\breve{\mathbf{M}}^*),\mathcal{S}_2(\breve{\mathbf{M}}^*),\mathcal{S}_3(\breve{\mathbf{M}}^*)\big)$ of wavelength indices $t=30,90,150$ in 10-sources Urban spectro-polarimetric matrix $\breve{\mathbf{M}}^*$ .}\label{urbandata}
\end{figure}

\begin{figure} 
\includegraphics[height=10cm,width=\textwidth]{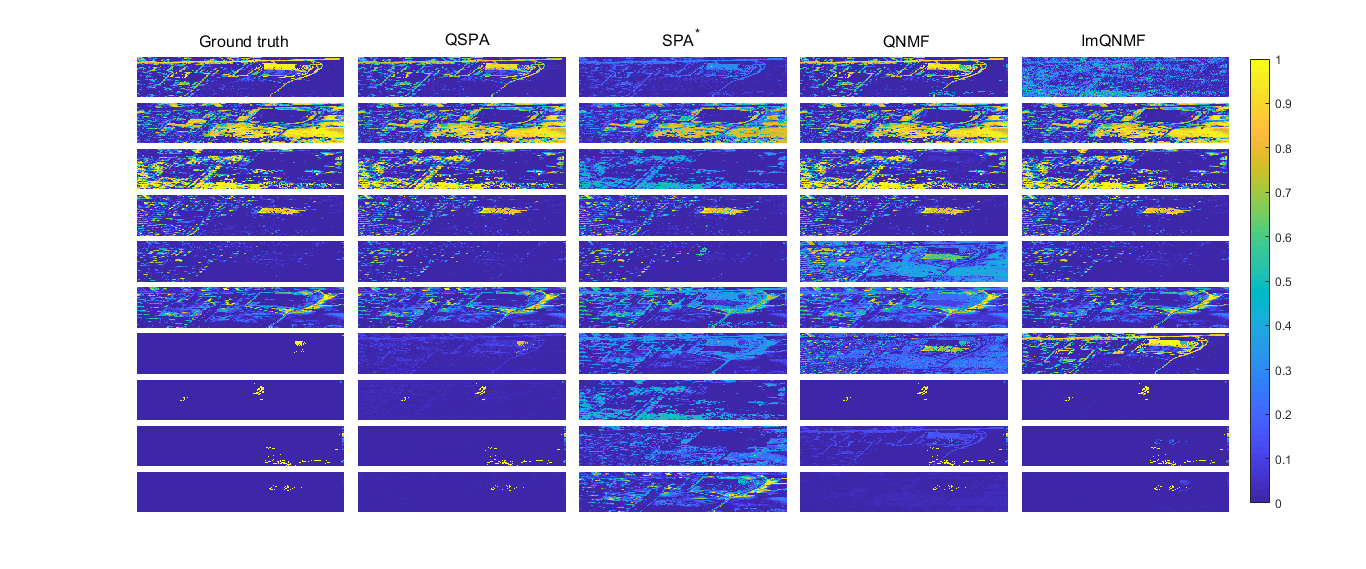}
\caption{The activation matrix $\mathbf{H}$ of all the methods in 10 sources spectro-polarimetric Urban dataset. }\label{H_urban_r10}
\end{figure}

\begin{figure} 
\includegraphics[height=9cm,width=\textwidth]{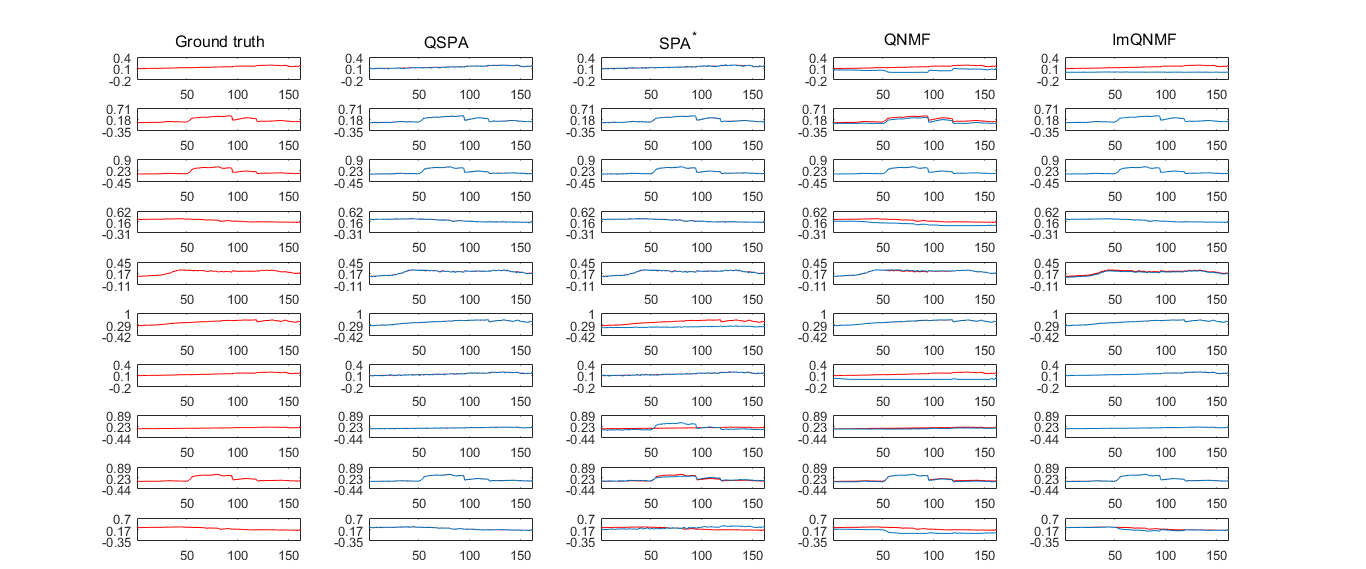}
\caption{$\mathcal{S}_0(\breve{\mathbf{W}})$ of all the methods in 10 sources spectro-polarimetric Urban dataset. The red line: ground truth; The blue line: the computed results.  }\label{S0_urban_r10}
\end{figure}

\begin{figure} 
\includegraphics[height=9cm,width=\textwidth]{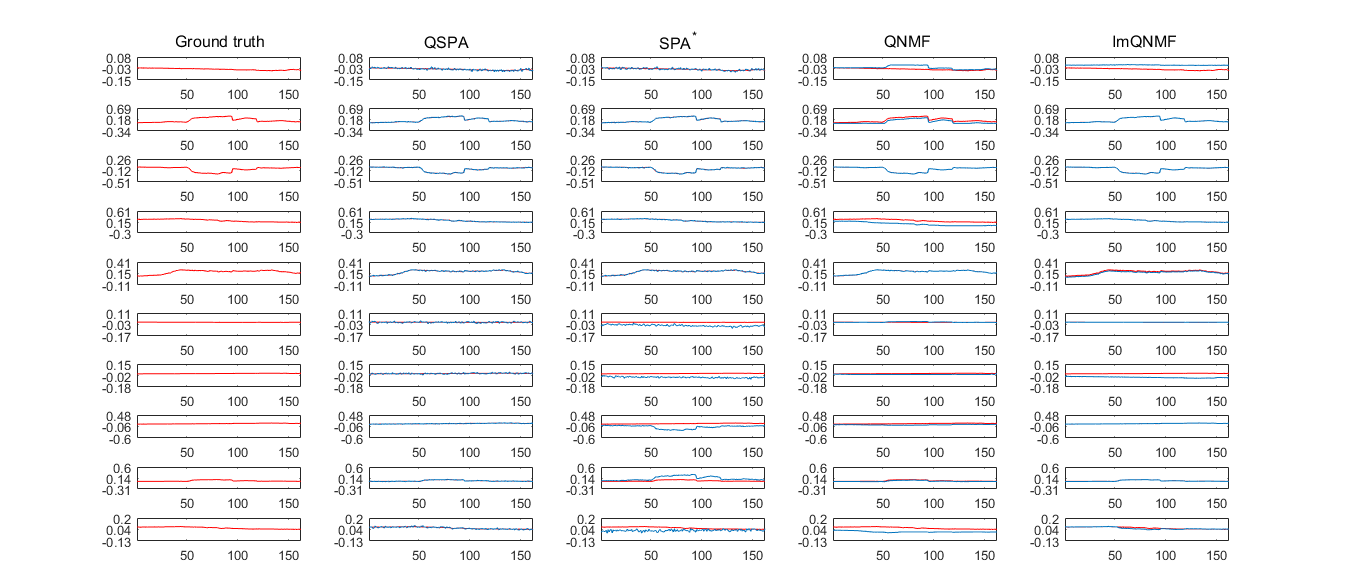}
\caption{$\mathcal{S}_1(\breve{\mathbf{W}})$ of all the methods in 10 sources spectro-polarimetric Urban dataset. The red line: ground truth; The blue line: the computed results.}\label{S1_urban_r10}
\end{figure}

\begin{figure} 
\includegraphics[height=9cm,width=\textwidth]{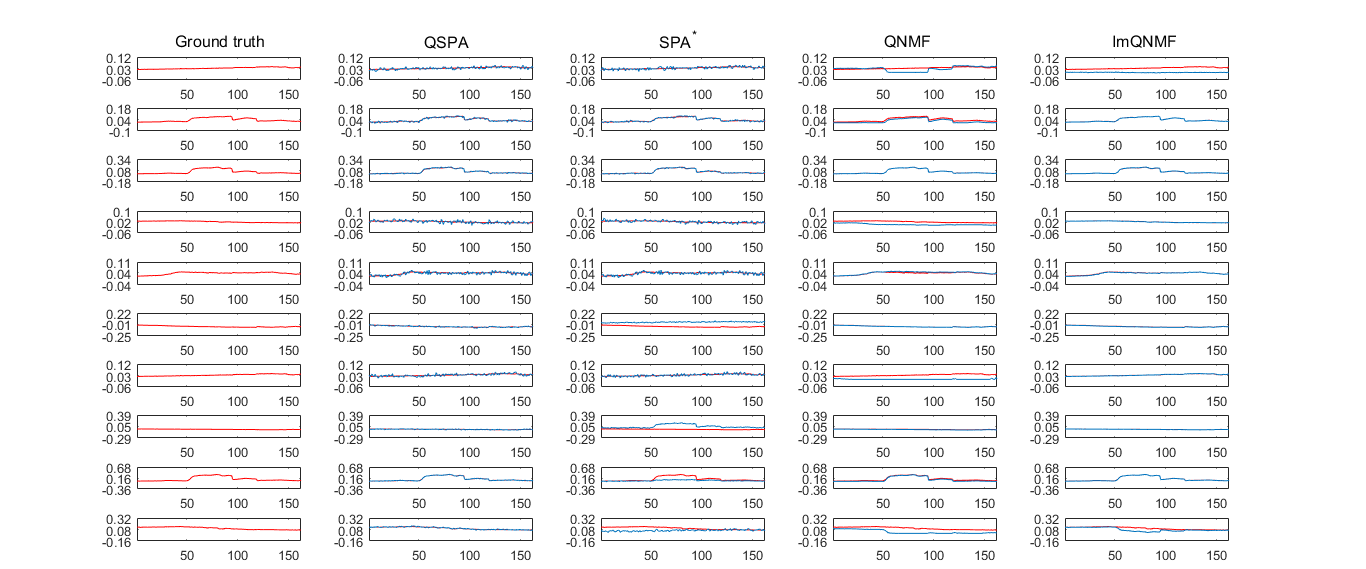}
\caption{$\mathcal{S}_2(\breve{\mathbf{W}})$ of all the methods in 10 sources spectro-polarimetric Urban dataset.  The red line: ground truth; The blue line: the computed results.}\label{S2_urban_r10}
\end{figure}

\begin{figure} 
\includegraphics[height=9cm,width=\textwidth]{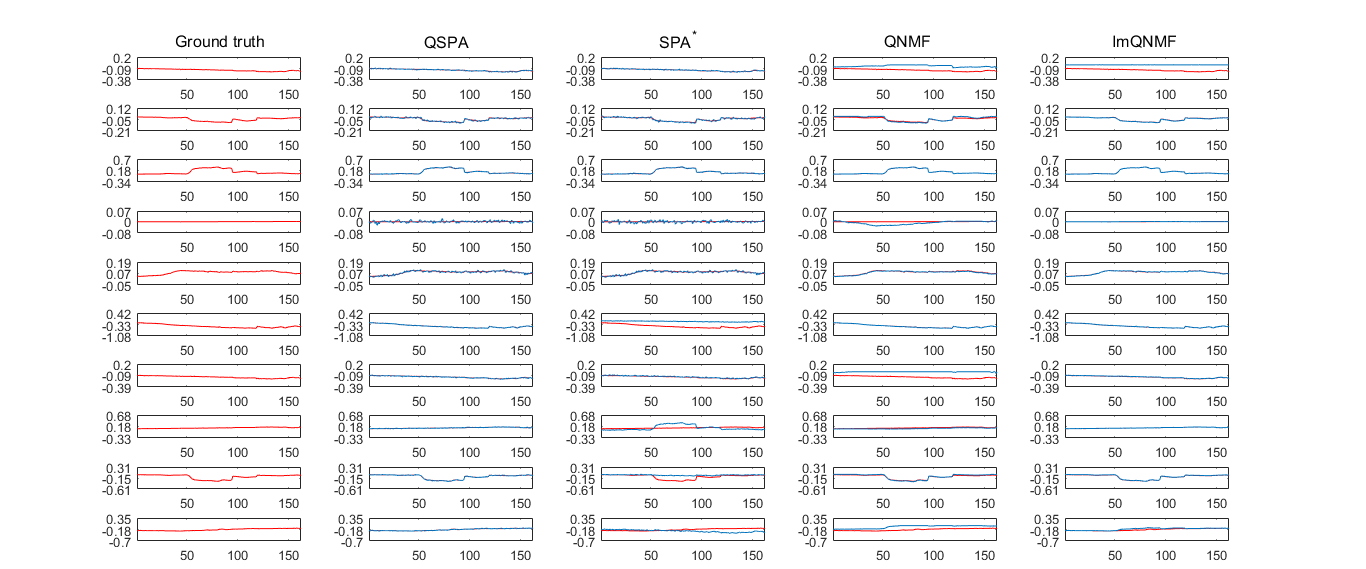}
\caption{$\mathcal{S}_3(\breve{\mathbf{W}})$ of all the methods in 10 sources spectro-polarimetric Urban dataset.  The red line: ground truth; The blue line: the computed results. }\label{S3_urban_r10}
\end{figure}

%We also test algorithms on a spectro-polarimetric data set simulated from Samson HSI dataset. The results are presented in Appendix.

We now make a conclusion on the advantages of QSPA in the application of spectro-polarimetric blind unmixing.  QSPA can well identify the underlying objects and sources from spectro-polarimetric data, even in the situation of the same intensity but foreign objects. Its computational time is competitive.

\section{Conclusion}
In this paper, we introduced separability into the quaternion matrix factorization and referred the problem to as separable quaternion  matrix factorization (SQMF). We studied some properties of SQ-matrices that can be decomposed by SQMF. We showed SQMF is unique up to scaling and permutations.  To solve quaternion-valued sources of SQMF, we proposed a fast and efficient method called quaternion successive projection algorithm (QSPA) extended from the successive projection algorithm (SPA). We have demonstrated that QSPA can correctly identity the quaternion-valued sources factor in noiseless cases. To compute the activation factor matrix, we provided a simple algorithm named quaternion hierarchical nonnegative least squares (QHNLS). We tested the algorithms on polarization image for image representation, and simulated spectro-polarimetric data set for blind unmixing.  The numerical results showed that the proposed method works promisingly.

Further work include a robust analysis of QSPA, and to design more efficient algorithms for the SQMF problem in the presence of noise. 

 \appendix
  
 \section{Appendix}
\subsection{Visual Results for Polarization Image Representation} 

Figs. \ref{block-selected30}-\ref{block-selected50} show $r=30,50$ important blocks determined by QSPA and SPA$^*$ respectively.  Visualization of reconstructed images by $r=30, 50$ features are shown in Figs.\ref{polarresult30}- \ref{polarresult50}. 

\begin{figure} 
\includegraphics[height=5cm,width=\textwidth]{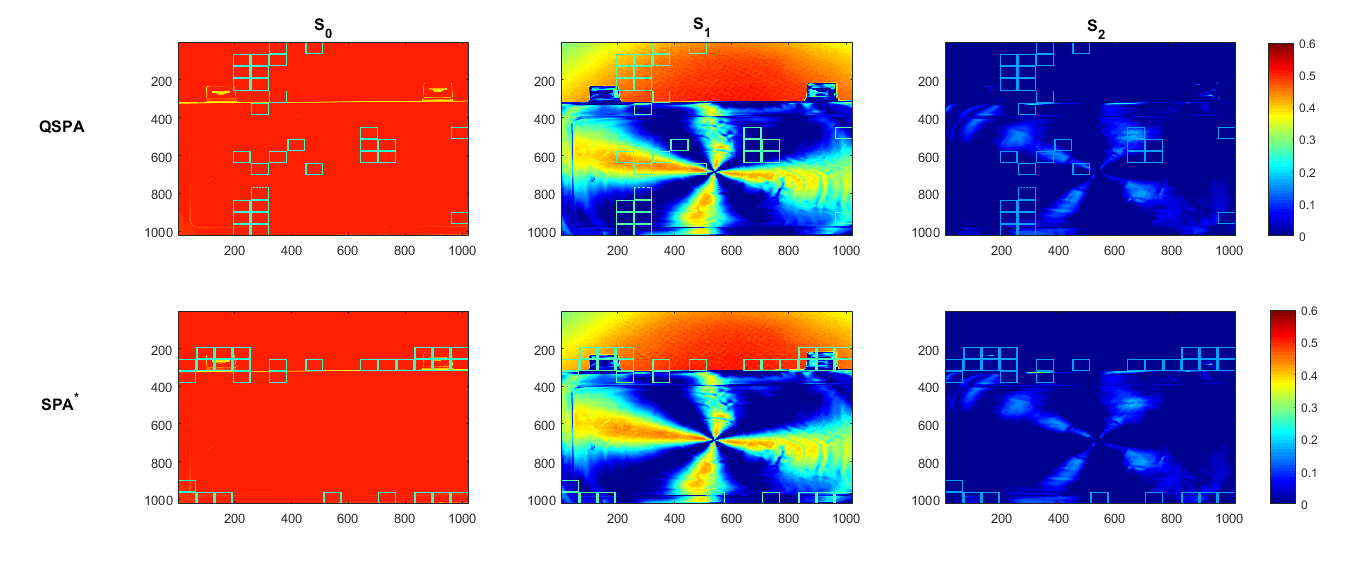}
\caption{30 key blocks identified by QSPA and SPA$^*$.}\label{block-selected30}
\end{figure}

\begin{figure} 
\includegraphics[height=8cm,width=\textwidth]{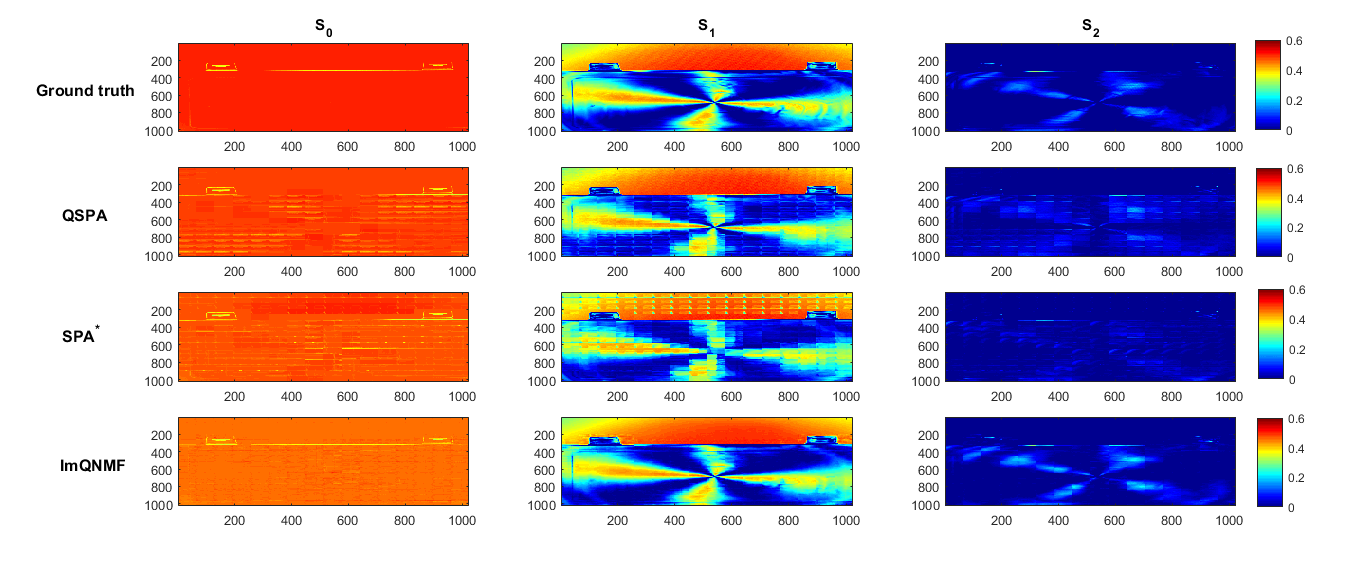}
\caption{Reconstruction results $(\mathcal{S}_0(\breve{\mathbf{T}}),\mathcal{S}_1(\breve{\mathbf{T}}),\mathcal{S}_2(\breve{\mathbf{T}}))$ of all methods $(r=30)$.}\label{polarresult30}
\end{figure}

\begin{figure} 
\includegraphics[height=5cm,width=\textwidth]{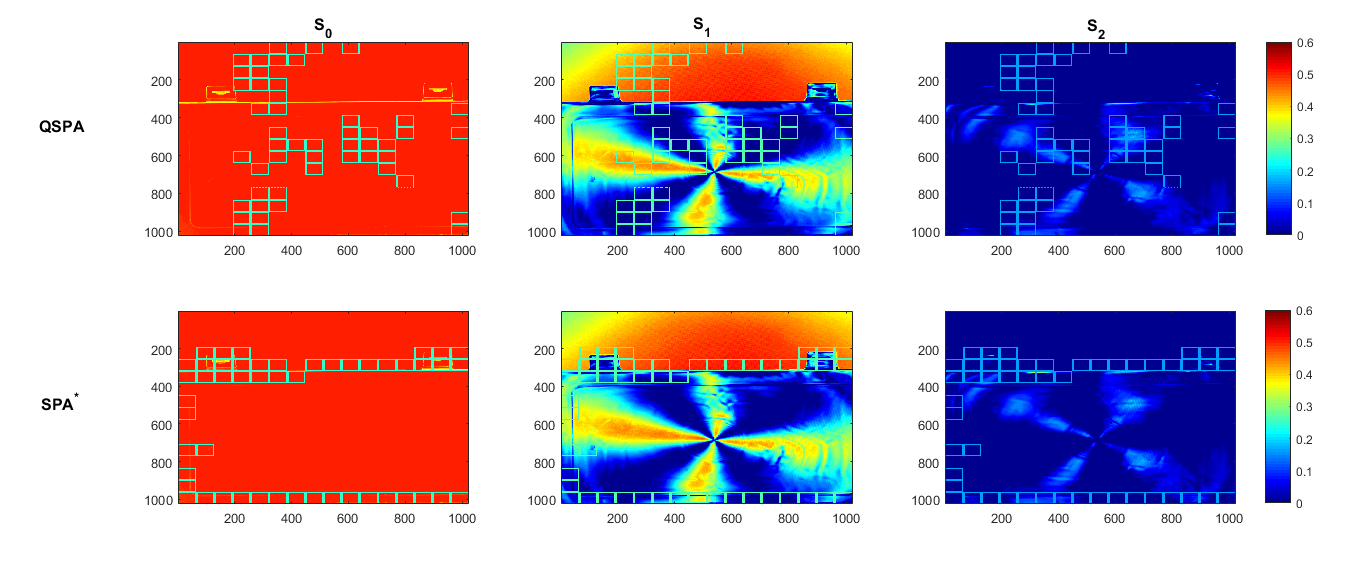}
\caption{50 key blocks identified by QSPA and SPA$^*$ .}\label{block-selected50}
\end{figure}

\begin{figure} 
\includegraphics[height=8cm,width=\textwidth]{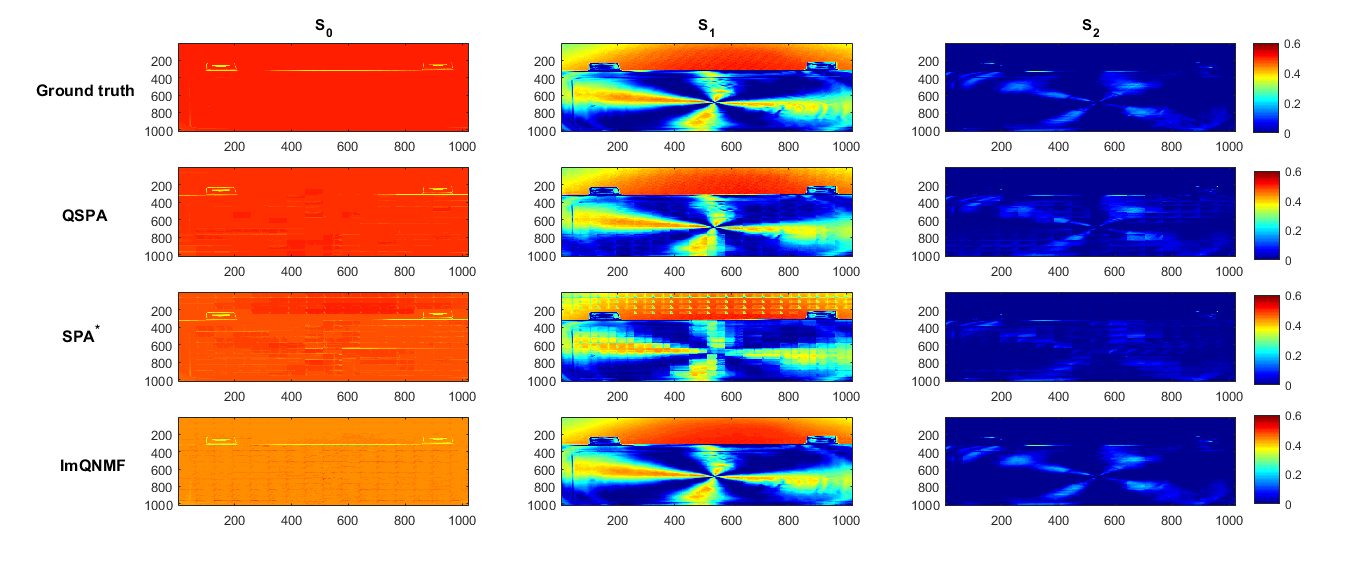}
\caption{Reconstruction results $(\mathcal{S}_0(\breve{\mathbf{T}}),\mathcal{S}_1(\breve{\mathbf{T}}),\mathcal{S}_2(\breve{\mathbf{T}}))$ of all methods $(r=50)$. }\label{polarresult50}
\end{figure}

 \subsection{Visual  Results on 6 sources Spectro-Polarimetric Urban Dataset}

We show visual results for one arbitrary group of simulated polarimetric parameters $\{\bm{\alpha},\bm{\beta}\}$ in Figs. \ref{urbandata_r6}-\ref{S3_urban_r6}. The simulated-polarimetric data $\breve{\mathbf{M}}^*=\breve{\mathbf{W}}^*\mathbf{H}^*$ at three distinct wavelengths indices $t=30,90,150$ are presented in 
Fig. \ref{urbandata_r6}. 
 We present the 6 sources and the corresponding activations factors from all the methods at the noise level $5\%$ in Figs. \ref{H_urban_r6}-\ref{S3_urban_r6}.  

\begin{figure} 
\includegraphics[height=8cm,width=\textwidth]{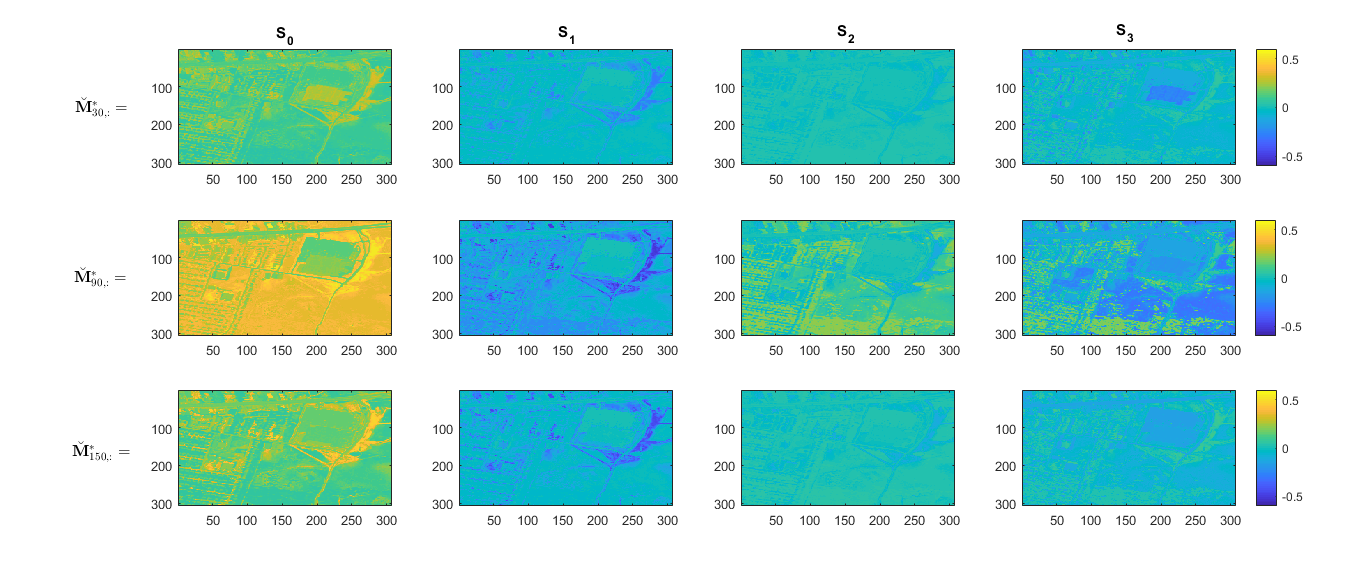}
\caption{2D intensity $\mathcal{S}_0(\breve{\mathbf{M}}^*)$ and polarization $\big(\mathcal{S}_1(\breve{\mathbf{M}}^*),\mathcal{S}_2(\breve{\mathbf{M}}^*),\mathcal{S}_3(\breve{\mathbf{M}}^*)\big)$ of wavelength indices $t=30,90,150$ in 6-sources Urban spectro-polarimetric matrix $\breve{\mathbf{M}}^*$.  }\label{urbandata_r6}
\end{figure}

\begin{figure} 
\includegraphics[height=8cm,width=\textwidth]{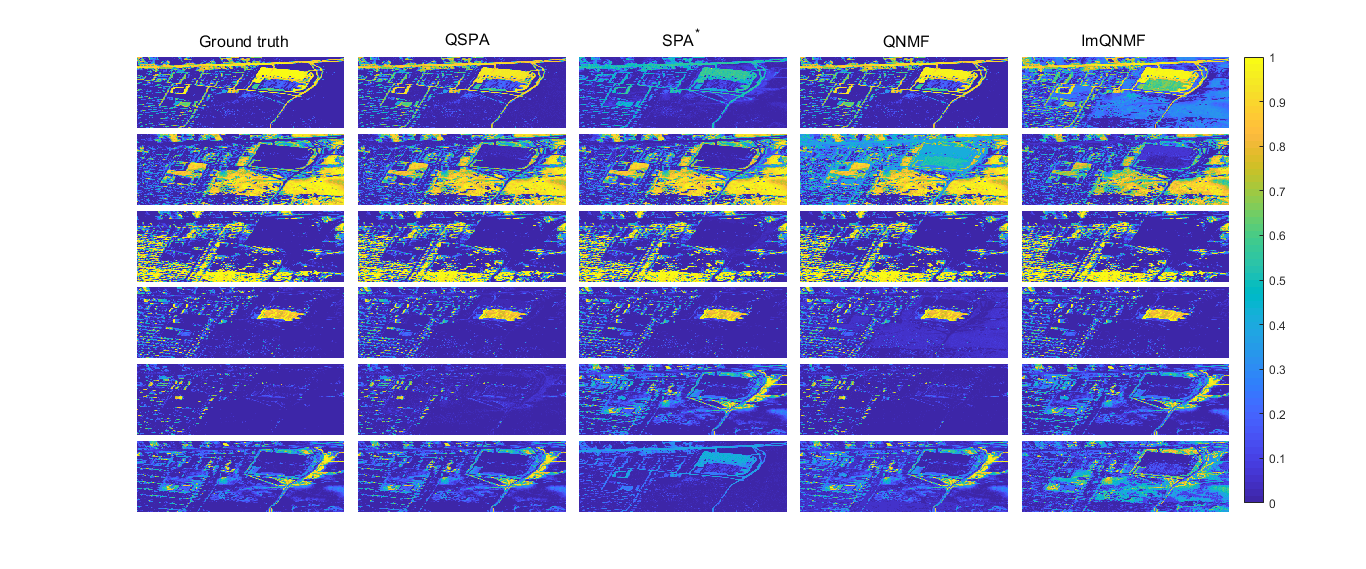}
\caption{The activation matrix $\mathbf{H}$ of all the methods in 6 sources spectro-polarimetric Urban dataset.}\label{H_urban_r6}
\end{figure}

\begin{figure} 
\includegraphics[height=8cm,width=\textwidth]{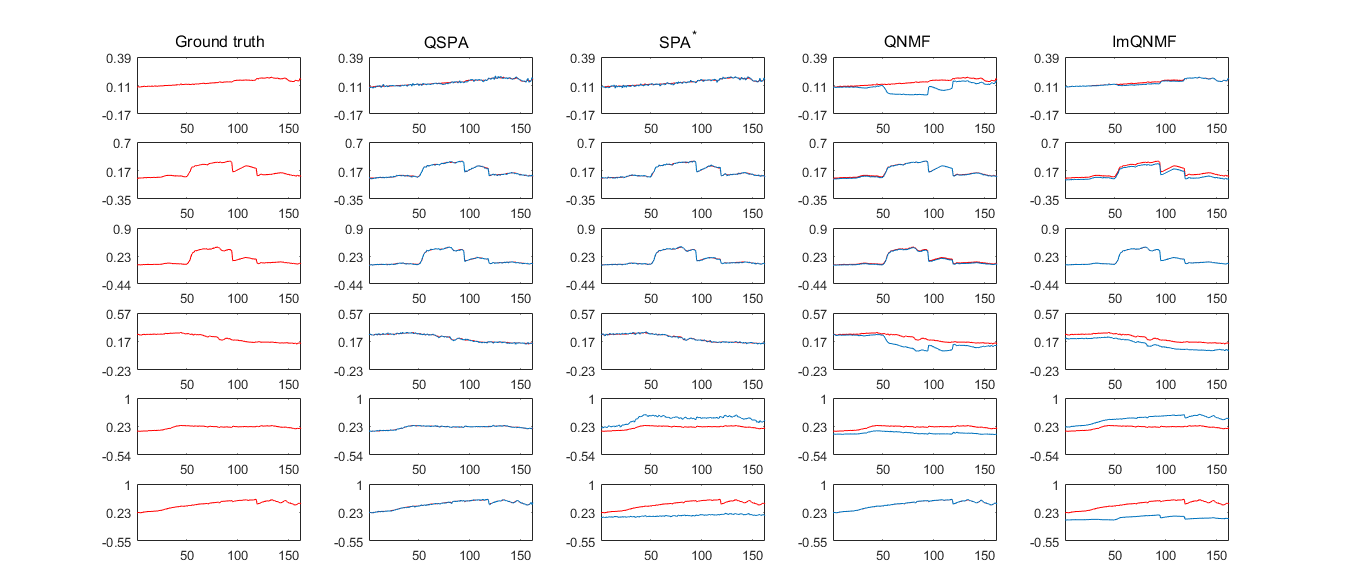}
\caption{$\mathcal{S}_0(\breve{\mathbf{W}})$ of all the methods in 6 sources spectro-polarimetric Urban dataset. The red line: ground truth; The blue line: the computed results. }\label{S0_urban_r6}
\end{figure}

\begin{figure} 
\includegraphics[height=8cm,width=\textwidth]{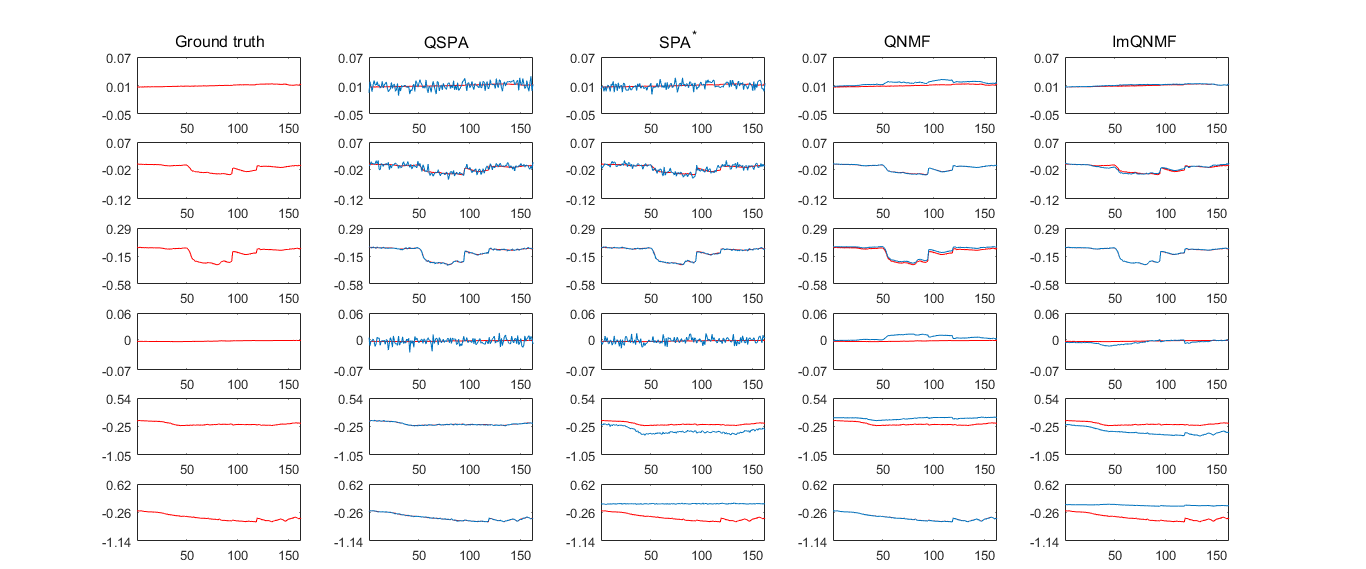}
\caption{$\mathcal{S}_1(\breve{\mathbf{W}})$ of all the methods in 6 sources spectro-polarimetric Urban dataset. The red line: ground truth; The blue line: the computed results.}\label{S1_urban_r6}
\end{figure}

\begin{figure} 
\includegraphics[height=8cm,width=\textwidth]{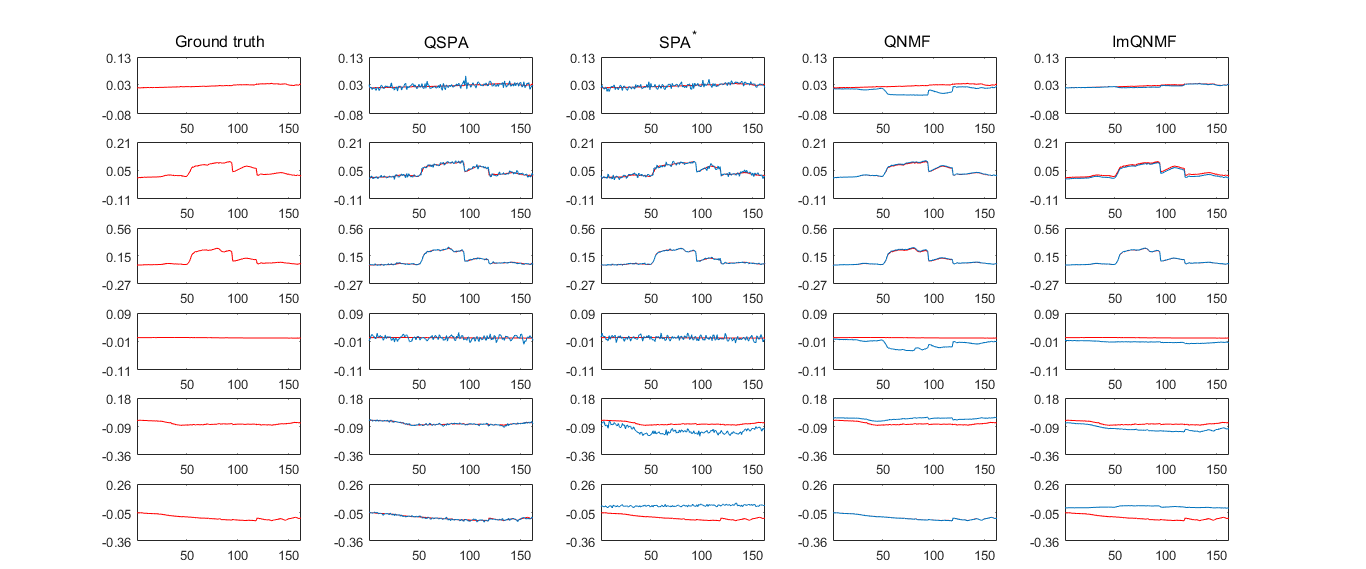}
\caption{$\mathcal{S}_2(\breve{\mathbf{W}})$ of all the methods in 6 sources spectro-polarimetric Urban dataset.  The red line: ground truth; The blue line: the computed results.}\label{S2_urban_r6}
\end{figure}

\begin{figure} 
\includegraphics[height=8cm,width=\textwidth]{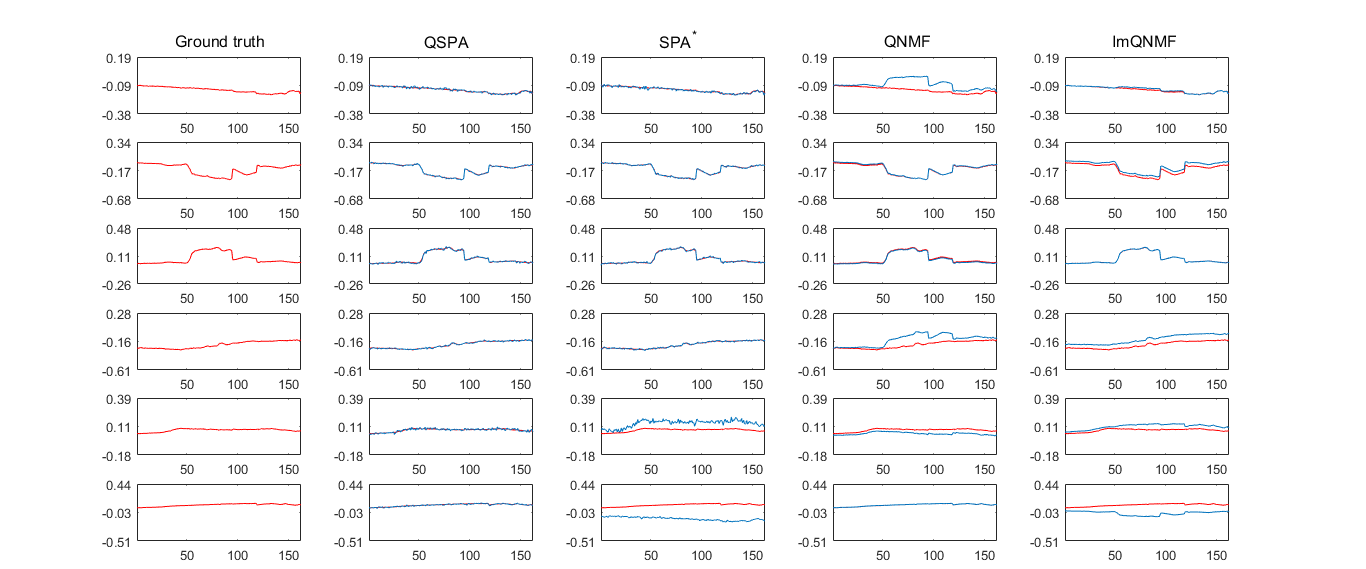}
\caption{$\mathcal{S}_3(\breve{\mathbf{W}})$ of all the methods in 6 sources spectro-polarimetric Urban dataset. The red line: ground truth; The blue line: the computed results.}\label{S3_urban_r6}
\end{figure}

%It is nice to observe in these figures that QSPA performs very well, obtaining better factors than the others. This validates the results in Table \ref{table:Urbansource06}. Since QSPA and SPA$^*$ determine $\breve{\mathbf{W}}$ from the noisy data matrix $\breve{\mathbf{M}}$, that leads to the results that the source matrices $\breve{\mathbf{W}}$ from both methods have noise compared to the ground truth $\breve{\mathbf{W}}^*$. However, it is  good to see that the $\breve{\mathbf{W}}$ from QSPA is very close to  ground truth $\breve{\mathbf{W}}^*$. 

\bibliographystyle{abbrv}
\bibliography{nmfref,ref} 
 \end{document}